%% file: arxiv.tex
\documentclass[11pt]{article} 
\usepackage[utf8]{inputenc}
\usepackage[T1]{fontenc}

\usepackage[top=3cm, bottom=3cm, left=3.17cm, right=3.17cm]{geometry}

\usepackage[round]{natbib}
\usepackage{framed}
\usepackage{amsthm}
\usepackage{amsmath}
\usepackage{amsfonts}
\usepackage{amssymb}
\usepackage{enumerate}
\usepackage{bbm}
\usepackage{graphicx}
\usepackage{subcaption}
\usepackage{float}
\usepackage{color}
\usepackage{caption}
\usepackage{microtype}

\usepackage{hyperref}

\usepackage[capitalise]{cleveref}

\usepackage[dvipsnames]{xcolor}
\newtheorem{assumption}{Assumption}

\hypersetup{
    colorlinks = true,
    citecolor=blue,
    urlcolor=blue,
    linkcolor=blue
}

\input{arxiv_macros}
\usepackage{notations}

\makeatletter
\def\set@curr@file#1{\def\@curr@file{#1}} 
\makeatother
\usepackage[load-configurations=version-1]{siunitx} 

\title{Near-optimal estimation of smooth transport maps with kernel sums-of-squares}

\author{ Boris Muzellec $^{\star *}$ $\quad$ Adrien Vacher $^{\circ *}$ $\quad$ Francis Bach $^{\star *}$   \\ Fran\c{c}ois-Xavier Vialard $^\circ$ $\quad$ Alessandro Rudi $^{\star *}$
}

\date{
{\small 
		    $^\circ$ LIGM, Université Gustave Eiffel, CNRS \\
			$^*$ INRIA Paris, 2 rue Simone Iff, 75012, Paris, France \\
			$^\star$ ENS - Département d’Informatique de l’École Normale Supérieure, \\
		$^\star$ PSL Research University, 2 rue Simone Iff, 75012, Paris, France \\
		\texttt{boris.muzellec@inria.fr, adrien.vacher@u-pem.fr, francis.bach@inria.fr, francois-xavier.vialard@u-pem.fr, alessandro.rudi@inria.fr}
}\\
}

\begin{document}

\maketitle

\begin{abstract}
    It was recently shown that under smoothness conditions, the squared Wasserstein distance between two distributions could be efficiently computed with appealing statistical error upper bounds. However, rather than the distance itself, the object of interest in statistical and machine learning applications is the underlying optimal transport map. Hence, computational and statistical guarantees need to be obtained for the estimated maps themselves. In this paper, we propose the first tractable algorithm for which the statistical $L^2$ error on the maps nearly matches the existing minimax lower-bounds for smooth map estimation. Our method is based on solving the semi-dual formulation of optimal transport with an infinite-dimensional sum-of-squares reformulation, and leads to an algorithm which has dimension-free polynomial rates in the number of samples, with potentially exponentially dimension-dependent constants.
\end{abstract}

\section{Introduction}\label{sec:intro}
Optimal transport (OT) provides a principled method to compare probability distributions, by finding the optimal way of coupling one to another based on a cost function on their supports. This optimization problem yields two useful quantities: the transport cost itself, which corresponds to the Wasserstein distance when the ground cost is a distance, and the minimizer, which is a map that pushes forward the first distribution onto the second, known as the transport map. While OT has gained attention lately in the statistics and machine learning communities due to the properties of the Wasserstein distance, transportation maps are playing an increasingly important role in data sciences. Indeed, many applications such as generative modeling \citep{arjovsky2017wasserstein,salimans2018improving,bernton2017inference,makkuva2020optimal,onken2021ot}, domain adaptation \citep{courty2016optimal,courty2017joint}, shape matching \citep{su2015optimal,Feydy2017} or predicting cell trajectories \citep{schiebinger2019optimal, yang2020predicting} among others can be formulated as the problem of finding a map from a reference distribution to a target distribution. 

\par Over the past decades, efforts have particularly concentrated on the problem of computing  OT distances. 
Two cases must be distinguished: the case of discrete measures (i.e.,  measures supported on a finite number of points), and the general case, including e.g.\ measures that admit a density with respect to the Lebesgue measure. For discrete measures supported on $N$ points, OT distances and plans may be computed by solving a linear program (LP) in $O(N^3\log(N))$ time using the network simplex algorithm~\citep[see, e.g.,][]{ahyja1993network}. Using entropic optimal transport \citep{cuturi2013sinkhorn} as a proxy, this computational cost can be further reduced to $O(N^2)$.
However, efficient methods for general measures remain to be found. The naive approach, which consists in using the OT distance between samples of the distributions (the so-called plugin estimator), fails as the dimension grows. Indeed,
the sharpest statistical bounds known for the plugin estimator require $\varepsilon^{-\frac{2}{d}}$ samples to achieve precision $\varepsilon$ \citep{chizat2020faster}. Yet, theoretical estimators were recently derived in the case where the problem is smooth \citep{weed2019estimation,hutter2019minimax}, yielding rates of estimation of the OT distance that improve as the smoothness grows. However, the estimators presented in those works are not algorithmically tractable. Very recently, \citet{vacher2021dimensionfree} closed this statistical-computational gap by designing an estimator of the squared Wasserstein distance relying on kernel sums-of-squares~(\citet{marteau2020non}, and in particular \cite{rudi2020global}) that may be computed with polynomial dimension-free rates in the number of samples, but with a constant term that is potentially exponential in the dimension.

\par On the other hand, fewer results on computationally efficient statistical estimators of optimal maps are available in the literature.
Compared to the problem of estimating OT distances, the estimation of an optimal transport map based on samples is particularly difficult since the estimated maps need to be evaluated on out-of-sample points. 
When smoothness assumptions are made, one can design theoretical statistical estimators of the transport maps that are statistically minimax optimal with respect to the smoothness for the $L^2$ error \citep{hutter2019minimax}, but such estimators cannot be computed in practice as they require projecting on the space of smooth and strongly convex functions, which is NP-hard. In a few empirical studies, the potentials were explicitly parameterized either by a neural network \citep{seguy2018maps} or by a Gaussian kernel \citep{genevay2016stochastic}. However, neither computational nor statistical guarantees are provided in those works. Likewise, \citet{paty2020convexity} recently managed to derive an explicit algorithm when the potentials are assumed to have smooth gradients and to be strongly convex, but the corresponding rate of approximation is not known.
A recent line of research studies estimators of optimal maps that rely on barycentric projection, which can be developed in regularized and non-regularized settings \citep{gunsilius_2021,deb2021rates,pooladian2021entropic}. These estimators can be used without any further assumption on the source and target measures. Here again, either the rates do match minimax rates of estimation of \citet{hutter2019minimax} without computationally feasible estimators, or the rates are not optimal. In particular, the estimator of \citet{pooladian2021entropic} is based on entropic regularization and is computationally friendly, but only works for low regularity of the maps, and is not minimax optimal.
Likewise, \cite{manole2021plugin} provide estimators for transportation maps that achieve minimax statistical rates both in smooth and non-smooth regimes. In the smooth regime, which is the setting that we consider in the present paper, \citeauthor{manole2021plugin} propose a map estimator for smooth distributions that are supported on the torus. This estimator is obtained by first performing sample-based wavelet density estimation for the source and target distributions, and then computing the optimal transport map between the estimated densities. While \citeauthor{manole2021plugin} show that this estimator achieves the minimax rates of \citet{hutter2019minimax}, it remains untractable due to the difficulty of computing the OT map between the two smooth distributions obtained from wavelet density estimation.

\paragraph{Contributions.} The contributions of this paper are twofold. First, relying on the sum-of-squares (SoS) tight reformulation of OT proposed by \cite{vacher2021dimensionfree}, we bridge the statistical-computational gap on the estimation of the optimal potentials. Second, we propose improvements on the algorithms proposed by \cite{vacher2021dimensionfree} to enhance the practicality of SoS for optimal transport. 
\begin{enumerate}
    \item \textbf{An estimator for smooth OT potentials.} Relying on an SoS representation of Brenier dual constraints, we design an algorithm to compute an estimator of the OT potentials based on samples. Remarkably, the complexity of this algorithm has polynomial rates in the number of samples (with potentially exponential constants w.r.t.\ the underlying dimension). For this estimator, we derive a statistical upper bound on the $L^2$ error on the resulting maps that nearly matches the lower bound of \citet{hutter2019minimax} when the smoothness is high. To achieve these nearly optimal rates, we refine the regularization that was proposed by \cite{vacher2021dimensionfree} and carry out a careful analysis to bound the solutions of the empirical problem. The central argument consists in using the local strong convexity of the so-called semi-dual formulation of optimal transport. It enables us to relate the $L^2$ error both to the quality of the (scalar) approximation of the original problem and to the norm of the empirical minimizers.
    In turn, this allows us to obtain a statistical rate for the $L^2$ error with well-chosen regularizers.
    \item \textbf{Algorithmic improvements.} We propose several improvements on the algorithms introduced by \cite{vacher2021dimensionfree}.
    First, as our objective is to compute an estimator of the OT potentials themselves, our algorithm should provide convergence guarantees for the minimizing solution, and not only the minimum. To do so, we reformulate the objective as a strongly convex optimization problem by using a square norm regularization instead of the trace, and by replacing hard constraints with quadratic penalties. A second advantage of this strongly convex relaxation is that it allows using first order optimization methods such as accelerated gradient descent, which are more scalable than the Newton methods employed by \cite{vacher2021dimensionfree} whose cost quickly blows up with the number of samples. Next, we further reduce computational costs by introducing a Nyström approximation of the features in the dual constraints. Finally, we propose a criterion to choose the hyperparameters on which our estimator relies, which was lacking in the work of \citeauthor{vacher2021dimensionfree}
\end{enumerate}

\section{Notations and background}

Let $(\mu, \nu)$ be two probability distributions on two bounded domains $X, Y \subset \mathbb{R}^d$. We study the squared Wasserstein distance, whose dual Kantorovitch formulation is given by
 \begin{align}
 \begin{split}
    \label{EqDualKantorovitchOT}
    W_2^2(\mu, \nu)  =   \sup_{u,v \in C(\R^d)} ~&~~ \langle u, \mu \rangle  + \langle v, \nu \rangle \\
    \textrm{subject to}  ~&~~  \frac{\|x-y\|^2}{2} \geq u(x) + v(y), ~~ \forall (x,y) \in X \times Y \, ,
\end{split}
\end{align}
where $C(\mathbb{R}^d)$ is the space of continuous functions on $\mathbb{R}^d$. For convenience, we shall use the equivalent Brenier formulation, where the Kantorovitch potentials $u, v$ are replaced by the Brenier potentials $f(\cdot):= \frac{\|\cdot\|^2}{2} - u(\cdot), g(\cdot):= \frac{\|\cdot\|^2}{2} - v(\cdot)$ 
 \begin{align}
 \begin{split}
    \label{EqDualBrenierOT}
    \OT(\mu, \nu)  =   \inf_{f,g \in C(\R^d)} ~&~~ \langle f, \mu \rangle  + \langle g, \nu \rangle \\
    \textrm{subject to}  ~&~~ f(x) + g(y) \geq \langle x, y \rangle, ~~ \forall (x,y) \in X \times Y \, .
\end{split}
\end{align}
These two quantities are related by $W_2^2(\mu, \nu) = \langle \frac{\|.\|^2}{2}, \mu + \nu \rangle - \OT(\mu, \nu)$. 
\par Problem \eqref{EqDualBrenierOT} (as well as problem \eqref{EqDualKantorovitchOT}) is delicate to solve numerically due to the non-negativity constraint which has to be satisfied on a continuous set. For instance, a feasible idea consists in sampling the inequality constraint on a finite set and trying to extend it to the continuous set. Unfortunately, this strategy may only leverage Lipschiztness, even if the functions involved are smoother than Lipschitz, yielding an approximation rate of order $n^{-\frac{1}{2d}}$, $2d$ being the dimension of $X \times Y$. The method proposed by \cite{vacher2021dimensionfree} overcomes this difficulty and is able to leverage the smoothness of the potentials. Following their work, we require assumptions on the support and the smoothness of the densities themselves.


\begin{assumption}[$m$-times differentiable densities]\label{assum:measures}
Let $m, d \in \NN$. Let $\mu, \nu \in {\cal P}(\R^d)$.
\begin{enumerate}[a)]
    \item \label{assum:support_sets} $\mu, \nu$ have densities. Their supports, resp. $X, Y \subset \RR^d$ are convex, bounded and open with a Lipschitz boundary;
    \item \label{assum:smooth_densties} the densities are finite, bounded away from zero, with Lipschitz  derivatives up to order $m$.
\end{enumerate}
\end{assumption}

Using Cafarelli's regularity theory \citep{philippis2013mongeampre}, \Cref{assum:measures} ensures that the Brenier potentials have a similar order of differentiability. In particular, defining the Sobolev space of order $m$ over $Z\subset \mathbb{R}^d$
\begin{equation}
    H^{m}(Z) := \Big\{f \in L^2(Z) ~|~ \|f\|_{H^{m}}:= \sum_{|\alpha| \leq m} \|D^\alpha f\|_{L^2(Z)} < \infty \Big\} \, ,
\end{equation}
the optimal Brenier potentials $(f_*, g_*)$ belong to the Sobolev space $H^{m+2}(X)$ and $H^{m+2}(Y)$ respectively \citep{philippis2013mongeampre}. When $m > d/2 - 2$, the Sobolev embedding theorem gives that these spaces are continuously embedded in the space of continuous functions and thus are reproducing kernel Hilbert spaces~\citep[RKHS, see e.g.][]{paulsen2016introduction,steinwart2008support}.

\par The approach proposed by \citeauthor{vacher2021dimensionfree} to estimate the Wasserstein distance is based on kernel Sum-Of-Squares (SoS) and the tools introduced by \cite{rudi2020global} to deal with optimization problems subject to a dense set of inequalities. The procedure consists in two steps: (1) show that the optimization problem is equivalent to a problem where the inequality constraint is substituted by an equality constraint with respect to an SoS term; (2) consider an empirical version of the resulting problem with only a finite number of equality constraints. \cite{rudi2020global} shows, for the case of non-convex optimization, that this procedure is adaptive to the degree of differentiability of the objective function leading to rates that overcome the curse of dimensionality for very smooth objectives. \cite{vacher2021dimensionfree} shows that for the case of the problem in \eqref{EqDualBrenierOT}, the two steps correspond to the following \cref{EqTightReformBrenierOT} and \cref{eq:EmpiricalDualBrenier} as reported in the following theorem and below.

\begin{theorem}[\citet{vacher2021dimensionfree}]
\label{thm:SoS-OT}
Let $\mu \in \mathcal{P}(X), \nu \in \mathcal{P}(Y)$ satisfy \Cref{assum:measures} and assume that $m>d+1$. Then, problem \eqref{EqDualBrenierOT} can be reformulated as
\begin{equation}
     \begin{split}
    \label{EqTightReformBrenierOT}
    \OT(\mu, \nu)  =  & \underset{\substack{f,g \in C(\R^d) \\ A \in \mathbb{S}^{+}(H^{m}(X \times Y))}}{\inf} ~~ \langle f, \mu \rangle  + \langle g, \nu \rangle \\
    & \textrm{subject to} ~  f(x) + g(y) = \langle x, y \rangle + \langle \phi(x,y), A\phi(x,y) \rangle_{H^{m}(X \times Y))} \, ,
\end{split}
\end{equation}
where $\mathbb{S}^{+}(H^{m}(X \times Y))$ is the set of linear, positive, self-adjoint operators on $H^{m}(X \times Y)$.
\end{theorem}
The key contribution of this representation theorem is to replace the inequality constraint by an equality constraint which is easier to deal with, as proposed and analyzed in \cite{rudi2020global}. Given access to samples $x_1, \cdots, x_n \sim \mu$ and $y_1, \cdots, y_n \sim \nu$ with associated empirical distributions $\hat{\mu}$ and $\hat{\nu}$, we solve for $\lambda_1, \lambda_2 \geq 0$ the empirical problem
\begin{align}\label{eq:EmpiricalDualBrenier}
\begin{split}
    & \min_{\substack{f \in H^{m+2}(X), g \in H^{m+2}(Y), \\ A \in \mathbb{S}_{+}(H^{m}(X\times Y))}} ~~~  \langle f, \hat{\mu} \rangle + \langle g, \hat{\nu} \rangle + \la_1 \tr(A) + \la_2(\|f\|^2_{H^{m+2}(X)} + \|g\|^2_{H^{m+2}(Y)}) \\
    & \textrm{subject to}  ~~~ \forall j \in [n], ~~ f(x_j) + g(y_j) - \langle x_j, y_j \rangle = \scal{\phi(x_j,y_j)}{A \phi(x_j,y_j)}_{H^{m}(X\times Y)} \, ,
\end{split}
\end{align}
where $\phi: X\times Y \mapsto H^{m}(X\times Y)$ is the feature map of $H^{m}(X\times Y)$. Using the techniques of \citet{marteau2020non} and \citet{rudi2020global}, the authors show that problem \eqref{eq:EmpiricalDualBrenier} that in the case where $m>2d$, the energy of the empirical potentials can be controlled with high probability by setting $\la_1=\la_2 \sim \frac{1}{\sqrt{n}}$, leading to a dimension-free approximation of $\OT(\mu, \nu)$ at a $\frac{1}{\sqrt{n}}$ rate with high probability, with a computational complexity of $O(n^{3.5} \log(\frac{1}{\varepsilon}))$, for a precision of $\varepsilon$. 

\par As problem \eqref{eq:EmpiricalDualBrenier} is strongly convex with respect to the potentials $f, g$, the uniqueness of the empirical potentials $\hat{f}_n, \hat{g}_n$ is ensured. Further, their existence may also be proven. However, the question remains of recovering a rate of convergence of the empirical potentials  toward the optimal Brenier potentials $(f_*, g_*)$ with respect to some norm. \citet{hutter2019minimax} prove that when $d>3$, any estimator of the transport map $\hat{T}_n$ can achieve at best an $L^2$ error that scales as 
\begin{equation}
    \mathbb{E}(\| \nabla f_* - \hat{T}_n \|^2_{L^2(\mu)}) \sim n^{-\frac{m + 1}{m + d/2}} \, .
\end{equation}
We prove in the next section that the solutions $(\hat{f}_n, \hat{g}_n)$ of the problem \eqref{eq:EmpiricalDualBrenier} can nearly match this rate up to an exponent $\varepsilon$, where $\varepsilon$ can be chosen arbitrary close to $0$, for well-chosen regularization parameters $\la_1, \la_2$. 

\section{Nearly optimal minimax rates}\label{sec:rates}

The key to derive the nearly minimax optimal rates is to have estimates on $\la_1, \la_2$ that are sharper than $\frac{1}{\sqrt{n}}$. Indeed, in the first step of the proof we show an upper-bound of the form
\begin{equation*}
    (\la_1 + \la_2)(\|\hat{f}_n\|_{H^{m+2}(X)} + \|\hat{g}_n\|_{H^{m+2}(Y)} + \hat{A}_n) \lesssim \frac{1}{\sqrt{n}}(\| \nabla \hat{f}_ n - \nabla f_*\|^2_{L^2(\mu)} + \| \nabla \hat{g}_ n - \nabla g_*\|^2_{L^2(\nu)})^{\alpha} \, ,
\end{equation*}
where $\alpha$ is positive, decreases with the dimension $d$ to zero and increases with the smoothness $m$ and $(\hat{f}_n, \hat{g}_n, \hat{A}_n)$ are minimizers of problem \eqref{eq:EmpiricalDualBrenier}. In order to control in turn the errors $\| \nabla f - \nabla f_*\|^2_{L^2(\mu)}, \| \nabla g - \nabla g_*\|^2_{L^2(\nu)}$ with the regularizers $\la_1, \la_2$, we need extra convexity. To this end, we introduce the semi-dual formulation of optimal transport \citep{brenier1987decomposition} which replaces the potential pair $(f,g)$ by the couple $(f,f^*)$ where $f^*$ is the Fenchel-Legendre transform $f^*(y) = \sup_{x \in X}  x^\top y  - f(x)$. It reads
\begin{equation}
    \OT(\mu, \nu) = \inf_{f \in C(\R^d)} ~~~ J_{\mu, \nu}(f), \,
\end{equation}
where $J_{\mu, \nu}(f) = \langle f, \mu \rangle + \langle f^*, \nu \rangle$. We simply denote it by $J(f)$ when no confusion is possible. As shown in the next lemma, the functional $J(f)$ gains stronger convexity in comparison with the linear objective of Brenier's formulation.
\begin{lemma}\label{lemma:j_strg_cvx}
The semi-dual objective $J(f)$ is convex in $f \in C(X)$. Assuming that there exists an optimal potential $f_*$ such that $\nabla f_*$ pushes $\mu$ onto $\nu$ and that $f$ is a convex $C^1$ function with $M$-Lipschitz gradient, we have
\begin{equation}
    \| \nabla f - T_* \|_{L^2(\mu)}^2 \leq 2 M (J(f) - J(f_*)) \, .
\end{equation}
\end{lemma}
Note that there is no assumption on the measures $\mu$ and $\nu$ themselves. However, we assume the Lipschitz smoothness of the gradient of $f$ and the existence of an optimal map, which is ensured by Brenier's theorem if $\mu$ has density w.r.t. the Lebesgue measure. A similar result is proven by \cite{hutter2019minimax} and we include for completeness a simple proof in Apppendix \ref{AppendixAdditionalProofs}. Thanks to this gain of convexity, we obtain an upper bound of the errors by the regularizers $\la_1, \la_2$ of the form
\begin{equation}
    \| \nabla \hat{f}_n - \nabla f_*\|^2_{L^2(\mu)} +  \| \nabla \hat{g}_n - \nabla g_*\|^2_{L^2(\nu)} \lesssim (\la_1 + \la_2) \frac{1}{\sqrt{n}} \, .
\end{equation}
Thanks to these two connected bounds, we can calibrate the regularizers $\la_1, \la_2$ and eventually obtain rates sharper than $\frac{1}{\sqrt{n}}$.

\begin{theorem} Let $\delta, \varepsilon \in ]0, 1[^2$ and assume that the regularizers are given by 
\begin{equation}
    \lambda_n^1 = \lambda_n^2 = \lambda_n = \biggl(\frac{\log(\frac{2}{\delta})^2}{n}\biggr)^{\frac{m+1}{m+d/2+\varepsilon}} + C_1\biggl(\frac{\log(\frac{n}{\delta})}{n}\biggr)^{\frac{m-d}{2d}} \, ,
\end{equation}
where $C_1$ is a constant that does not depend on $n$ and $\delta$. Denoting $\hat{f}_n$ (resp. $\hat{g}_n$) the transport potential associated to $\mu$ (resp. $\nu$) in problem \eqref{eq:EmpiricalDualBrenier}, we have with probability at least $1 - \delta$ that for $n \geq n_0(X, Y, d, m)$, 
\begin{equation}
    \| \nabla \hat{f}_n - \nabla f_* \|_{L^2(\mu)}^2 + \| \nabla \hat{g}_n - \nabla g_* \|_{L^2(\nu)}^2 \leq C_2 \lambda_n\, ,
\end{equation}
where $C_2$ is a positive constant that is independent from $n$ and $\delta$ but grows to infinity when $\varepsilon$ goes to $0$. In particular, when $m$ is sufficiently large, the minimax rate is nearly attained:
\begin{equation}
     \| \nabla \hat{f}_n - \nabla f_* \|_{L^2(\mu)}^2 + \| \nabla \hat{g}_n - \nabla g_* \|_{L^2(\nu)}^2 \leq C_2 \biggl(\frac{\log(\frac{2}{\delta})^2}{n}\biggr)^{\frac{m+1}{m+d/2+\varepsilon}} \, .
\end{equation}
\end{theorem}

\begin{proof}
Let us denote by $(\hat{f}_n, \hat{g}_n, \hat{A}_n)$ the solutions of the empirical problem \eqref{eq:EmpiricalDualBrenier} and by $(f_*, g_*, A_*)$ the solutions of the deterministic problem \eqref{EqTightReformBrenierOT}. We define the energy of the potentials  $R^2 = \| f_* \|_{H^{m+2}(X)}^2 + \| g_* \|_{H^{m+2}(Y)}^2 $ and its empirical counterpart $\hat{R}^2_n = \| \hat{f}_n \|_{H^{m+2}(X)}^2 + \| \hat{g}_n \|_{H^{m+2}(Y)}^2 $. For probability measures $(\alpha, \beta)$, we shall denote throughout the proof the linear objective of the dual as $\mathcal{E}_{\alpha, \beta}(f,g) = \langle f, \alpha \rangle + \langle g, \beta \rangle$. The strategy of the proof is as follows
\begin{enumerate}
    \item We use the optimality conditions of the empirical and deterministic OT problems to upper bound $\lambda_n$ and the norms of the empirical objects $(\hat{f}_n, \hat{g}_n, \hat{A}_n)$ by the errors $\|\nabla \hat{f}_n - f_*\|_{L^2(\mu)}, \|\nabla \hat{g}_n - g_*\|_{L^2(\nu)}$.
    \item We use the strong convexity of the semi-dual to upper bound in turn the errors by the regularizer $\lambda_n$ and the norms of $(\hat{f}_n, \hat{g}_n, \hat{A}_n)$.
    \item We obtain two connected bounds and we optimize over $\lambda_n$ to obtain sharp rates while maintaining the empirical objects $(\hat{f}_n, \hat{g}_n, \hat{A}_n)$ bounded.
\end{enumerate}

\paragraph{Upper bound on the regularizer.} The optimality conditions in the empirical problem \eqref{eq:EmpiricalDualBrenier} yield

\begin{equation}
    \label{eqEmpiricalOptimality}
    \lambda_n(\tr(\hat{A}_n) - \tr(A_*) + \hat{R}_n^2 - R^2)\leq \mathcal{E}_{\hat{\mu}, \hat{\nu}}(f_*, g_*) - \mathcal{E}_{\hat{\mu}, \hat{\nu}}(\hat{f}_n, \hat{g}_n) \, .
\end{equation}
Conversely, we wish to test the empirical potentials against the measures $\mu$, $\nu$ and use the optimality condition of the deterministic problem. 

\begin{proposition}
\label{PropAdmissibleEmpiricalPotentials}
There exist constants $C, G$ such that for all $0 <\delta < 1$, defining $\hat{\kappa}_{n, \delta} = C(\hat{R}_n + \tr(\hat{A}_n) + G)n^{-\frac{m-d}{2d}} \log(\frac{n}{\delta})^{\frac{m-d}{2d}}$ we have with probability at least $1 - \delta$ that for ${n\geq n_0(X, Y,d,m)}$, the empirical potentials $(\hat{f}_n, \hat{g}_n + \hat{\kappa}_{n, \delta})$ are admissible candidates:
\begin{equation}
    \forall (x, y) \in X \times Y, ~ \hat{f}_n(x) + \hat{g}_n(y) + \hat{\kappa}_{n, \delta} \geq x^\top y \, .
\end{equation}
\end{proposition}
The proof is left in Appendix \ref{AppendixAdditionalProofs}. We deduce that with probability at least $1 - \delta$
\begin{equation}
    \label{eqDetermOptimality}
    0 \leq \mathcal{E}_{\mu, \nu}(\hat{f}_n, \hat{g}_n) - \mathcal{E}_{\mu, \nu}(f_*, g_*) + \hat{\kappa}_{n, \delta} \, .
\end{equation}
Adding the equation above with the empirical optimality conditions \eqref{eqEmpiricalOptimality} yields
\begin{align*}
\begin{split}
     \lambda_n(\tr(\hat{A}_n) - \tr(A_*) + \hat{R}_n^2 - R^2) &
     \leq \mathcal{E}_{\hat{\mu}, \hat{\nu}}(f_*, g_*) - \mathcal{E}_{\hat{\mu}, \hat{\nu}}(\hat{f}_n, \hat{g}_n) + \mathcal{E}_{\mu, \nu}(\hat{f}_n, \hat{g}_n) - \mathcal{E}_{\mu, \nu}(f_*, g_*)  \\ &\qquad + \hat{\kappa}_{n, \delta}
\end{split}
     \\[2ex]
\begin{split}
    & = \langle f_* - \hat{f}_n, \hat{\mu} \rangle + \langle g_* - \hat{g}_n, \hat{\nu} \rangle + \langle \hat{f}_n - f_*, \mu \rangle + \langle \hat{g}_n - g_*, \nu \rangle \\ &\qquad + \hat{\kappa}_{n, \delta}
\end{split}
\\[2ex]
     & = \langle \hat{f}_n - f_*, \mu - \hat{\mu} \rangle + \langle \hat{g}_n - g_*, \nu - \hat{\nu} \rangle + \hat{\kappa}_{n, \delta}\,.
\end{align*}

Now we use the following lemma upper-bounding the linear forms of the r.h.s.\ with respect to the errors $\| \nabla \hat{f}_n - f_*\|_{L^2(\mu)}, \| \nabla \hat{g}_n - g_*\|_{L^2(\nu)}$.
 \begin{lemma}\label{lemma:scal_ub}
Let $(u,v) \in  H^{m+2}(X) \times H^{m+2}(X)$, $(\delta, \varepsilon) \in ]0, 1[ \times ]0, 1[$ and $\mu$ be a probability measure over $X$. If $\mu$ has a density on $X$ bounded away from zero, then there exists a constant $C$ independent of $n, \delta, u, v$ such that with probability at least $1 - \delta$, we have
 \begin{equation}
     \langle u - v, \mu - \hat{\mu} \rangle \leq C_{\mu} \frac{\log(\frac{2}{\delta})}{\sqrt{n}}\biggl( \Delta_{u, v, \mu}^{\frac{m+2 - d/2 - \varepsilon}{m+1}}R(u,v)^{\frac{d/2 + \varepsilon - 1}{m+1}} + \Delta_{u, v, \mu}\biggl) \, ,
 \end{equation}
 where $\Delta_{u, v, \mu} = \| \nabla u - \nabla v \|_{L^2(\mu)}$ and $R(u,v) = \| u \|_{H^{m+2}(X)} + \| v \|_{H^{m+2}(X)}$.
 \end{lemma}
 The proof of this lemma is left in the Appendix \ref{AppendixAdditionalProofs} and is mainly an application of the Gagliardo-Nirenberg inequality and the Poincaré inequality. Applying the lemma to $(f_*, \hat{f}_n, \mu)$ and $(g_*, \hat{g}_n, \nu)$, and using the fact that for $(\alpha, x,y) \in (\mathbb{R}^+)^3$, we have $x^\alpha + y^\alpha \leq 2 (x+y)^\alpha$, we recover that for $n\geq n_0(X,Y,d,m)$ we have with probability at least $1 - 3 \delta$ 
 \begin{align}\label{eq:p_ub_fine}
     \begin{split}
         \lambda_n(\tr(\hat{A}_n) - \tr(A_*) + \hat{R}_n^2 - R^2) \leq & ~ 2 C' \frac{\log(\frac{2}{\delta})}{\sqrt{n}}(\hat{R}_n + R)^{\frac{d/2 + \varepsilon - 1}{m+1}} a_n^{\frac{m+2 - d/2 - \varepsilon}{m+1}} \\
         ~  & \hspace*{2cm} +  2 C' \frac{\log(\frac{2}{\delta})}{\sqrt{n}}a_n + \hat{\kappa}_{n, \delta} \, , 
     \end{split}
 \end{align}
 where $C' = C_\mu + C_\nu$ and we denoted $a_n = \| \nabla \hat{f}_n - \nabla f_* \|_{L^2(\mu)} +  \| \nabla \hat{g}_n - \nabla g_* \|_{L^2(\nu)} $.
 
\paragraph{Bounding the error $a_n$.} First we bound the errors from above via the strong convexity of the semi-dual. Then, using the optimality conditions, we recover the same linear upper bound as in  the previous paragraph.

\paragraph{a) Semi-dual upper bound.}

A straightforward way to upper-bound the errors would be to apply Lemma \ref{lemma:j_strg_cvx} to the empirical potentials $\hat{f}_n, \hat{g}_n$. Unfortunately, there is no guarantee that the empirical potentials are convex. Instead, we define for $t \in [0, 1]$ the following interpolating potential
 \begin{equation}
     \tilde{f}_n(t)(.) = f_*(.) + t(\hat{f}_n(.) - f_*(.)) \, .
 \end{equation}
 Since $f_*$ and $g_*$ are mutual Legendre transforms and that they both have Lipschitz gradients, $(f_*, g_*)$ are strongly convex and we denote by $\gamma$ a strong convexity constant of the optimal Brenier potentials. Hence for $t \leq \frac{\gamma}{2 \| \hat{f}_n - f \|_{W^{2, \infty}(X)}}$, the interpolate $\tilde{f}_n(t)(.)$ is $\frac{\gamma}{2}$ strongly convex. Hence we choose
 \begin{equation}
     \hat{t}_f = \min \biggl ( 1,  \frac{\gamma}{2 \| \hat{f}_n - f \|_{W^{2, \infty}(X)}}\biggr) \, ,
 \end{equation}
 and we apply Lemma \ref{lemma:j_strg_cvx} to $\tilde{f}_n(\hat{t}_f)(.)$, which yields 
 \begin{equation}
     \| \nabla \tilde{f}_n(\hat{t}_f) - \nabla f_* \|_{L^2(\mu)}^2 \leq 2 \hat{L}_f( J( \tilde{f}_n(\hat{t}_f)) - J(f_*)) \, ,
 \end{equation}
 where $\hat{L}_f = \| \tilde{f}_n(\hat{t}_f) \|_{W^{2, \infty}(X)}$. By convexity of the semi-dual $J$, the r.h.s.\ is upper bounded by $2 \hat{L}_f \hat{t}_f (J(\hat{f}_n) - J(f_*))$ and the l.h.s.\ is equal to $\hat{t}_f ^2\|\nabla \hat{f}_n - \nabla f\|_{L^2(\mu)}^2$. Hence, we have
\begin{equation}\label{eq:strg_cvx_f}
    \|\nabla \hat{f}_n - \nabla f\|_{L^2(\mu)}^2 \leq \frac{2\hat{L}_f}{\hat{t}_f}(J(\hat{f}_n) - J(f_*)) \, .
\end{equation}
Finally, we upper-bound the quantity $\frac{\hat{L}_f}{\hat{t}_f}$.
\begin{proposition}\label{prop:grad_lip_ub}
Denoting $K_X$ the embedding constant such that $\forall u \in W^{2, \infty}(X), \|u\|_{W^{2, \infty}(X)} \leq K_X  \|u\|_{H^{m+2}(X)}$, the quantity $\frac{\hat{L}_f}{\hat{t}_f}$ can be upper bounded as 
\begin{equation}
    \frac{\hat{L}_f}{\hat{t}_f} \leq 2K_X (\| \hat{f}_n \|_{H^{m+2}(X)} + \| f_* \|_{H^{m+2}(X)}) \biggl( 1 + \frac{ \|f_* \|_{H^{m+2}(X)}}{\gamma}\biggr) \, .
\end{equation}
\end{proposition}
The proof is left in \Cref{AppendixAdditionalProofs}. Conversely, applying this result to $\hat{g}_n$, we obtain
\begin{equation}
    \frac{\hat{L}_g}{\hat{t}_g} \leq 2K_Y (\| \hat{g}_n \|_{H^{m+2}(Y)} + \| g_* \|_{H^{m+2}(Y)}) \biggl( 1 + \frac{ \|g_* \|_{H^{m+2}(Y)}}{\gamma}\biggr) \, ,
\end{equation}
where $K_Y$ is the embedding constant from $H^{m+2}(Y)$ to $W^{2, \infty}(Y)$. In particular
\begin{align}
\label{UbDistortion}
\begin{split}
    \frac{\hat{L}_f}{\hat{t}_f} +  \frac{\hat{L}_g}{\hat{t}_g} & \leq 2(K_X + K_Y)(\| \hat{f}_n \|_{H^{m+2}(X)} + \| \hat{g}_n \|_{H^{m+2}(Y)} + \| f_* \|_{H^{m+2}(X)} + \| g_* \|_{H^{m+2}(Y)}) \\
     &\qquad \times \biggl( 1 + \frac{\|f_* \|_{H^{m+2}(X)}) + \|g_* \|_{H^{m+2}(Y)}}{\gamma}\biggr)
\end{split}
      \\
     & \leq  2\sqrt{2}(K_X + K_Y)( \hat{R}_n + R )\biggl( 1 +  \frac{\sqrt{2}R}{\gamma}\biggr)  \, .
\end{align}

\paragraph{b) Linear upper bound.} Recall that the couple $(\hat{f}_n, \hat{g}_n + \hat{\kappa}_{n, \delta})$ are admissible candidates and in particular $J(\hat{f}_n) \leq \mathcal{E}_{\mu, \nu}(\hat{f}_n, \hat{g}_n + \hat{\kappa}_{n, \delta})$, which implies
\begin{equation}\label{eq;ub_j_raw_loose}
    J(\hat{f}_n) - J(f_*) \leq \mathcal{E}_{\mu, \nu}(\hat{f}_n, \hat{g}_n) -  \mathcal{E}_{\mu, \nu}(f_*, g_*) + \hat{\kappa}_{n, \delta} \, .
\end{equation}
Combining the equation above with the empirical optimality condition $0 \leq \mathcal{E}_{\hat{\mu}, \hat{\nu}}(f_*, g_*) -  \mathcal{E}_{\hat{\mu}, \hat{\nu}}(\hat{f}_n, \hat{g}_n) + \lambda_n (\tr(A_*) + R^2)$ \eqref{eqEmpiricalOptimality}, we obtain
\begin{equation} \label{eq:ub_j_raw_tight}
     J(\hat{f}_n) - J(f_*) \leq \langle f_* - \hat{f}_n, \mu - \hat{\mu} \rangle + \langle g_* - \hat{g}_n, \nu - \hat{\nu} \rangle + \lambda_n (\tr(A_*) + R^2) + \hat{\kappa}_{n, \delta} \, .
\end{equation}
As in the previous paragraph, we use Lemma \ref{lemma:scal_ub} to upper-bound the linear forms of the r.h.s., and denoting $a_n = \| \nabla \hat{f}_n - \nabla f_* \|_{L^2(\mu)} + \| \nabla \hat{g}_n - \nabla g_* \|_{L^2(\nu)}$, we obtain
\begin{equation} \label{eq:diff_semi_dual_ub}
\begin{split}
     J(\hat{f}_n) - J(f_*) \leq & 2 C' \frac{\log(\frac{2}{\delta})}{\sqrt{n}}\biggr[(\hat{R}_n + R)^{\frac{d/2 + \varepsilon - 1}{m+1}} a_n^{\frac{m+2 - d/2 - \varepsilon}{m+1}} + a_n\biggr] \\
     & + \lambda_n(\tr(A_*) + R^2) + \hat{\kappa}_{n, \delta}\, ,
\end{split}
\end{equation}
where $C'$ is a constant independent of $n, \delta$.

\paragraph{Combining the bounds on the error $a_n$.}

From the previous paragraphs, we obtain the following bounds
\begin{equation}
\label{eq:final_ub_1}
\begin{split}
    \lambda_n(\tr(\hat{A}_n) + \hat{R}_n^2) \leq & ~ 2 C' \frac{\log(\frac{2}{\delta})}{\sqrt{n}}[(\hat{R}_n + R)^{\frac{d/2 + \varepsilon - 1}{m+1}} a_n^{\frac{m+2 - d/2 - \varepsilon}{m+1}} + a_n] \\
     ~  & + \lambda_n(\tr(A_*) + R^2) + \hat{\kappa}_{n, \delta} \, ,
\end{split}
\end{equation}

\begin{equation}
\label{eq:final_ub_2}
\begin{split}
     a_n^2 \leq ~ 2(\hat{R}_n + R) K C'\biggl(&  \frac{\log(\frac{2}{\delta})}{\sqrt{n}}\biggr[(\hat{R}_n + R)^{\frac{d/2 + \varepsilon - 1}{m+1}} a_n^{\frac{m+2 - d/2 - \varepsilon}{m+1}} + a_n\biggr] \\
     & + \lambda_n(\tr(A_*) + R^2) + \hat{\kappa}_{n, \delta} \biggr) \, ,
\end{split}
\end{equation}
where we defined $K = \sqrt{2}(K_X + K_Y)\biggl( 1 +  \frac{2\sqrt{2}R}{\gamma}\biggr)$. The inequalities above mutually relate the empirical norms $\hat{R}_n, \tr(\hat{A}_n)$ and the error $a_n$. The next result shows that for a well-chosen regularizer $\lambda_n$, the empirical norms are bounded independently on $\delta$ and that for this choice of regularizer, the error $a_n^2$ is at most of the order $\lambda_n$.
\begin{proposition}\label{prop:asymptotic_behavior}
Assume that Equations \eqref{eq:final_ub_1} and \eqref{eq:final_ub_2} hold for $n \geq n_0$. If we set the regularizer as
\begin{equation}
    \lambda_n = \biggl(\frac{\log(\frac{2}{\delta})^2}{n}\biggr)^{\frac{m+1}{m+d/2+\varepsilon}} + C_1\biggl(\frac{\log(\frac{n}{\delta})}{n}\biggr)^{\frac{m-d}{2d}}   \, ,
\end{equation}
where $C_1$ is the constant independent of $n, \delta$, then $\hat{R}_n, \tr(\hat{A}_n)$ are bounded independently of $\delta$ and there exists a constant $C_2$ independent of $n, \delta$, such that
\begin{equation}
    a_n^2 \leq C_2 \lambda_n \, .
\end{equation}
\end{proposition}
The proof is presented in Appendix \ref{AppendixAdditionalProofs}. Hence setting $\lambda_n$ to the indicated value yields that with probability at least $1-3 \delta$, for $n\geq n_0(X,Y,d,m)$ we have
\begin{equation}
    a_n^2 \leq C \lambda_n \, ,
\end{equation}
with $C$ a constant independent of $n, \delta$.

\end{proof}

\paragraph{Discussion and remarks.}
Let us make a few remarks regarding the comparison of our results to those of \citet{hutter2019minimax} and \citet{vacher2021dimensionfree}.
\begin{remark}
Our bound does not exactly match the lower-bound $n^{-\frac{m+1}{m+ d/2}}$ of \citet{hutter2019minimax} because of the extra exponent $\varepsilon$. Indeed, we cannot take $\varepsilon = 0$ since the constant $C_2$ would then go to infinity. However, it would be interesting to recover the growth rate of $C_2$ when $\varepsilon$ goes to $0$ ; if it does not increase too fast, we could hope for a trade-off in $n$ on $\varepsilon$ and maybe recover the exact lower-bound.
\end{remark}

\begin{remark}
Note that the minimax rate is nearly attained only for the highly smooth case $m > 2d$, and that when $m$ goes toward $d$ we cannot even guarantee that empirical potentials converges to the original potentials. This gap compared with \citet{hutter2019minimax} is explained by the fact that in their case, the authors directly solve the semi-dual problem while in our case, we approximate the cost constraint through the SoS reformulation which can only be achieved when the smoothness is sufficiently high.
\end{remark}

\begin{remark}
Comparing ours results to \citet{vacher2021dimensionfree}, one may notice that we require lower values for the regularizers. In particular, one may wonder if this could lead to gains in terms of the estimation of the squared Wasserstein distance itself. This is not the case, as this effect is due to the fact that we are here considering the Brenier version of optimal transport. Indeed, it holds $W_2^2(\mu, \nu) = \langle \frac{\|.\|^2}{2}, \mu + \nu \rangle - \OT(\mu, \nu)$. Hence, to estimate $W_2^2(\mu, \nu)$ one must estimate the extra moment term, that fluctuates as $\frac{1}{\sqrt{n}}$.
\end{remark}


\section{Algorithms and Numerical Experiments}\label{sec:algos}
In this section, we provide and test algorithms to solve \cref{eq:EmpiricalDualBrenier} and compute the estimators $\hat{f}$ and $\hat{g}$. Following \cite{rudi2020global} and \cite{vacher2021dimensionfree}, we may rewrite \cref{eq:EmpiricalDualBrenier} using a finite-dimensional representation of $A$ that relies on finite-dimensional features $\Phi_i \in \RR^{\ell}, i \in [n]$, and a PSD matrix $\bB\in\pdm{\RR^{\ell}}$:
\begin{align}\label{eq:EmpiricalDualBrenier_finite_dim}
\begin{split}
    & \min_{\substack{f \in H^{m+2}(X), g \in H^{m+2}(Y), \\ \bB \in \pdm{\RR^\ell}}} ~~~  \langle f, \hat{\mu} \rangle + \langle g, \hat{\nu} \rangle + \la_1 \tr \bB + \la_2(\|f\|^2_{H^{m+2}(X)} + \|g\|^2_{H^{m+2}(Y)}) \\
    & \textrm{subject to}  ~~~ \forall j \in [\ell], ~~ f(\tilde{x}_j) + g(\tilde{y}_j) - \dotp{\tx_j}{\ty_j}  = \Phi_j^T\bB\Phi_j \, .
\end{split}
\end{align}
We refer to \cite{vacher2021dimensionfree} for the definition of $\Phi_j, j \in [\ell]$ and the derivation of \cref{eq:EmpiricalDualBrenier_finite_dim}.
Note that compared to \cref{eq:EmpiricalDualBrenier} we have two sets of samples that do not necessarily coincide: $x_1 \dots x_n \in X$ and $y_1, \dots, y_n \in Y$ which are the support points of the empirical distributions $\hat{\mu}$ and $\hat{\nu}$, and $(\tx_1, \ty_i) \dots (\tx_\ell, \ty_\ell) \in X \times Y$ which are used to subsample the dual constraints. 
One can then apply the method proposed by \cite{vacher2021dimensionfree} and solve  \eqref{eq:EmpiricalDualBrenier}
by using barrier methods on its dual formulation:
\begin{align}\label{eq:dual_barrier}
\begin{split}
    &\underset{\gamma \in \RR^\ell}{\min}~~ \frac{1}{4\la_2} \gamma^T \bQ \gamma - \frac{1}{2\la_2} \sum_{j=1}^\ell \gamma_j z_j + \frac{q^2}{2\la_2} - \frac{\delta}{\ell}\log \det \left(\sum_{j=1}^\ell \gamma_j \Phi_j\Phi_j^T + \la_1 \eye_\ell\right)\\
    &\mbox{ subject to } ~~~  \sum_{j=1}^\ell \gamma_j \Phi_j\Phi_j^T + \lambda_1 \eye_\ell \succeq 0,
    \end{split}
\end{align}
where $\bQ \in \RR^{\ell \times \ell}$ is defined as $\bQ_{i,j} = k(\tx_i, \ty_i) + k(\tx_i, \ty_i), i, j \in [\ell]$, $z_j = \tx_j \cdot \ty_j + w_{\hat{\mu}}(\tx_j) + w_{\hat{\nu}}(\ty_j)$, $q^2 = \|w_{\hat{\mu}}(\tx_j)\|_\hh^2 + \| w_{\hat{\nu}}(\ty_j)\|_\hh^2$, and $\delta > 0$ is the barrier parameter.\footnote{Compared to \cite{vacher2021dimensionfree}, the definition of $z$ above is slightly different. This is due to the fact that we consider the ``Brenier'' formulation of dual OT (an infimum with an inner product cost), instead of the Kantorovich formulation (a supremum with a quadratic cost).}
The authors showed that solving this problem to precision $\varepsilon$ has a $O(n^{3.5}\log \frac{1}{\varepsilon})$ computational cost.
This method has two drawbacks: first, its cost is prohibitive when the number of samples becomes large, which is necessary to ensure better statistical approximation, and second, there are no guarantees on the convergence to the minimizers $\hat{f}$ and $\hat{g}$ as only the objective value is guaranteed to converge to the optimal value. Both issues can be alleviated by introducing a strongly convex and unconstrained relaxation of problem \eqref{eq:EmpiricalDualBrenier}.

Let us now consider a variant of problem \eqref{eq:EmpiricalDualBrenier} with a squared norm penalty (instead of a trace penalty) on $\bB$:
\begin{align}\label{eq:w_hat_norm}
\begin{split}
    \inf_{\substack{f \in \hhx, g\in \hhy, \\  \bB \in \pdm{\RR^\ell}}} ~~~ & \scal{f}{\hat{w}_\mu}_\hhx + \scal{g}{\hat{w}_\nu}_\hhy + \frac{\la_1}{2} \|\bB\|_F^2 + \la_2(\|f\|^2_\hhx + \|g\|^2_\hhy) \\
    ~~~ \textrm{subject to} ~~~ &\forall j \in [\ell], ~~ f(\tilde{x}_j) + g(\tilde{y}_j) - \dotp{\tx_j}{\ty_j} =  \Phi_j^T\bB\Phi_j.
\end{split}
\end{align}
\begin{lemma}\label{lemma:frob_pen_dual}
\Cref{eq:w_hat_norm} admits the following dual problem:
 \begin{align}\label{eq:w_hat_norm_dual}
\begin{split}
      \underset{\gamma \in \RR^\ell}{\min}~~ \frac{1}{4\lambda_2} \gamma^T {\bf Q} \gamma - \frac{1}{2\lambda_2}\sum_{j=1}^\ell \gamma_j z_j + \frac{1}{2\la_1} \| (- \sum_{i=1}^\ell \gamma_i \Phi_i \Phi_i^T)_+\|_F^2 + \frac{q^2}{4\la_2},
 \end{split}
\end{align}
where $(\bM)_+$ denotes the positive part of $\bM \in \RR^\ell$, i.e.\ its projection on $\pdm{\RR^\ell}$.
\end{lemma}
Proof in \Cref{sec:algo_proofs}, \cpageref{proof:frob_pen_dual}. Although the bounds described in \Cref{sec:rates} and by \cite{vacher2021dimensionfree} rely on a trace control of $\bB$, we conjecture that similar results could be derived using a square Hilbert-Schmidt norm control.

\paragraph{Faster algorithms with strong convexity.}

Replacing the trace penalty with a square norm penalty led to the {\em unconstrained} convex dual problem \eqref{eq:w_hat_norm_dual}, without adding a barrier term. However, this problem is not necessarily strongly convex. Indeed, $\bQ$ may have arbitrarily small eigenvalues, and the positive part term vanishes when its argument is a negative matrix. 
We now propose a strongly convex relaxation of \eqref{eq:w_hat_norm_dual}, obtained by replacing the constraints in \eqref{eq:w_hat_norm} with a quadratic penalization. That is, for $\delta > 0$, we consider:
\begin{align}\label{eq:w_hat_norm_pen}
\begin{split}
    \inf_{\substack{f \in \hhx, g\in \hhy, \\  \bB \in \pdm{\RR^d}}} ~~~ & \Big\lbrace\scal{f}{\hat{w}_\mu}_\hhx + \scal{g}{\hat{w}_\nu}_\hhy + \frac{\la_1}{2} \|\bB\|_F^2 + \la_2(\|f\|^2_\hhx + \|g\|^2_\hhy) \\ 
    & \hspace*{2cm} + \frac{\delta}{2\ell} \sum_{j=1}^\ell  (f(\tilde{x}_j) + g(\tilde{y}_j) - \dotp{\tilde{x}_j}{\tilde{y}_j} -  \Phi_j^T\bB\Phi_j)^2\Big\rbrace.
\end{split}
\end{align}
Let us derive the corresponding dual problem.

\begin{proposition}\label{prop:relaxed_dual}
Strong duality holds, and problem \eqref{eq:w_hat_norm_pen} admits the following dual formulation:
\begin{align}\label{eq:w_hat_norm_pen_dual} 
\begin{split}
    \inf_{\gamma \in \RR^\ell} ~~~ \frac{1}{4\lambda_2} \gamma^T {\bf Q} \gamma - \frac{1}{2\lambda_2}\sum_{j=1}^\ell \gamma_j z_j + \frac{1}{2\la_1} \| (- \sum_{i=1}^\ell \gamma_i \Phi_i \Phi_i^T)_+\|_F^2 + \frac{\ell}{2\delta} \|\gamma\|^2+ \frac{q^2}{4\la_2},
\end{split}
\end{align}
with the following primal-dual relations
\begin{align}\label{eq:fg_opt}
    \begin{split}
    f &= \frac{1}{2\lambda_2}(\sum_{i=1}^\ell \gamma_i \phi_X(\tilde{x}_i) - \hat{w}_\mu) \\
    g &= \frac{1}{2\lambda_2}(\sum_{i=1}^\ell \gamma_i \phi_Y(\tilde{y}_i) - \hat{w}_\nu).
    \end{split}
\end{align}
This problem is strongly convex, with constant $\alpha = \frac{\ell}{\delta} +  \frac{1}{2 \la_2}\lambda_{\min}(\bQ)$, and smooth, with constant $L = \frac{\ell}{\delta} + \frac{1}{2\la_2} \la_{\max}(\bQ) + \frac{1}{\la_1} \lambda_{\max}(\bK\circ \bK)$ (where $\bK = \Phi^T\Phi$).
\end{proposition}
Formulation \eqref{eq:w_hat_norm_pen_dual} is a tight relaxation of \eqref{eq:w_hat_norm_dual}, as the dual of the constrained problem may be recovered by taking $\delta \rightarrow \infty$.
As mentioned above, for most kernels the spectrum of $\bQ$ decays rapidly to zero. Therefore, problem \eqref{eq:w_hat_norm_dual} may not be considered strongly convex for practical purposes. On the other hand, problem \eqref{eq:w_hat_norm_pen_dual} has a strong convexity constant that is bounded from under by $\frac{\delta}{\ell}$, thanks to the quadratic regularization term on $\gamma$ that appears in \cref{eq:w_hat_norm_pen_dual}.
Hence, we may solve \eqref{eq:w_hat_norm_pen_dual} using accelerated gradient descent algorithms. The bottleneck of the computation is the positive part, which is obtained by computing the SVD of a $\ell\times\ell$ matrix in $O(\ell^3)$ time. The size of this matrix can be reduced by considering a low-rank approximation of $\bK$. 

Note however that relaxing the constraints in \eqref{eq:w_hat_norm} with a quadratic penalization introduces a gap with the theory developed in \Cref{sec:rates}. However, we expect that bounds could be obtained in the case of quadratic penalties, much like soft constraints compared with hard constraints in kernel ridge regression. We leave this for future work. 

\paragraph{Nyström approximation.}

The computational bottleneck of solving \eqref{eq:w_hat_norm_pen_dual} is forming the matrix $\sum_{i=1}^\ell \gamma_i \Phi_i \Phi_i^T$ and computing the SVD at each iteration, both  in $O(\ell^3)$ time. This computational cost can be reduced using a Nyström approximation~\citep[see e.g.][]{williams2001using} of the kernel matrix $\bK$. Indeed, such an approximation is of the form
\begin{align}
    \bK^{nys} = \bV\bW^{-1}\bV^T,
\end{align}
where $\bW \in \RR^{r\times r}$ is the submatrix corresponding to $r < \ell$ randomly sampled columns (and corresponding rows) of $\bK$, and $\bV \in \RR^{\ell\times r}$. Hence, writing $\bW = \bL^T\bL$ with $\bL \in \RR^{r \times r}$, we may consider features $\Phi^{nys}_{i} \in \RR^r, i \in [\ell]$ corresponding to the columns of 
\begin{align}
\bR^{nys} \defeq \bL^{—T}\bV^T,
\end{align}
verifying $(\bR^{nys})^T\bR^{nys} = \bK^{nys}$. Replacing $\bK$ with $\bK^{nys}$ in the derivation of the dual leads to  

\begin{align}\label{eq:w_hat_norm_pen_dual_low_rank} 
\begin{split}
    \inf_{\gamma \in \RR^\ell} ~~~& \frac{1}{4\lambda_2} \gamma^T {\bf Q} \gamma - \frac{1}{2\lambda_2}\sum_{j=1}^\ell \gamma_j z_j + \frac{1}{2\la_1} \| (- \sum_{j=1}^\ell \gamma_j \Phi_j^{(nys)} (\Phi_j^{nys})^T)_+\|_F^2 + \frac{\ell}{2\delta} \|\gamma\|^2+ \frac{q^2}{4\la_2}.
\end{split}
\end{align}
Forming the matrix $\sum_{i=1}^\ell\gamma_i \Phi_i^{(nys)} (\Phi_i^{nys})^T$ costs $O( r^2\ell)$ (compared to $O(\ell^3)$ in \eqref{eq:w_hat_norm_pen_dual}), and computing its SVD costs $O(r^3)$ (compared to $O(\ell^3)$). This brings down the cost of a gradient step to $O(r^3 + r^2\ell)$, compared to $O(\ell^3)$ without using approximations.

The empirical results in \Cref{fig:2D_gaussian,fig:4D_gaussian,fig:8D_gaussian} tend to show that one can pick quite small ranks while retaining good performance: indeed, in those experiments the maximum employed rank is $100$, whereas we increase the number of samples up to $1000$. This is coherent with results on the Nyström approximation applied to statistical learning~\citep{rudi2015less}. However, the effect of Nyström approximation is currently not measured in our statistical analysis.

\paragraph{Selecting the hyperparameters.}

Contrary to the optimization setting~\citep{rudi2020global}, there is no direct way of assessing the quality of the estimator $\hat{OT}$ (e.g.\, of the ``smaller is better'' type) for a given choice of hyperparameters, as the algorithm may output values that are larger or smaller than the true value of $\OT$. Alternatively, we may use the output potentials to estimate transportation maps $\hat{T}_1, \hat{T}_2$ from $\mu$ to $\nu$ and $\nu$ to $\mu$, and compare the mapped points to the target distributions. That is, we compute 
\begin{align}\label{eq:forward_backward_map}
\begin{split}
\hat{T}_1 &: x \rightarrow \nabla f(x) =  \frac{1}{2\la_2} \left(\sum_{j=1}^{\ell}\gamma_j \nabla_x k(\tx_j, x) - \nabla_x \hat{w}_\mu(x) \right)\\
\hat{T}_2 &: x \rightarrow \nabla f(x) =  \frac{1}{2\la_2} \left(\sum_{j=1}^{\ell}\gamma_j \nabla_x k(\ty_j, x) - \nabla_x \hat{w}_\nu(x) \right),
\end{split}
\end{align}
and select the hyperparameters for which
\begin{equation}\label{eq:mmd_criterion}
    \widehat{\MMD} \defeq \MMD\left(\frac{1}{n_\mu}\sum_{i=1}\delta_{\hat{T}_1 (x_i)}, \hat{\nu}\right) + \MMD\left(\frac{1}{n_\nu}\sum_{i=1}\delta_{\hat{T}_2 (y_i)}, \hat{\mu}\right)
\end{equation}
is the smallest, where MMD is the maximum mean discrepancy~\citep{gretton2012kernel} for the kernel $k$, defined as
\begin{equation}
    \MMD(\mu, \nu) \defeq \iint (k(x, x) -  2k(x, y) + k(y, y))\d(\mu\otimes\nu)(x, y).
\end{equation}
The choice of MMD as a criterion has two main motivations: first, it is a divergence that can be efficiently computed and whose plugin estimator converges at rate $O(n\smrt)$, and second, as observed by \cite{vacher2021dimensionfree} in Appendix F, when the filing samples $\tx_j, \ty_j, j \in [\ell]$ correspond to all pairs of samples in $\hat{\mu}$ and $\hat{\nu}$, the dual problem \eqref{eq:dual_barrier} may be reformulated as a regularized optimal transport problem with MMD marginal penalties.
\paragraph{Numerical results.}
\begin{figure}
    \centering
    \begin{subfigure}[b]{0.45\textwidth}
         \centering
         \includegraphics[width=\textwidth]{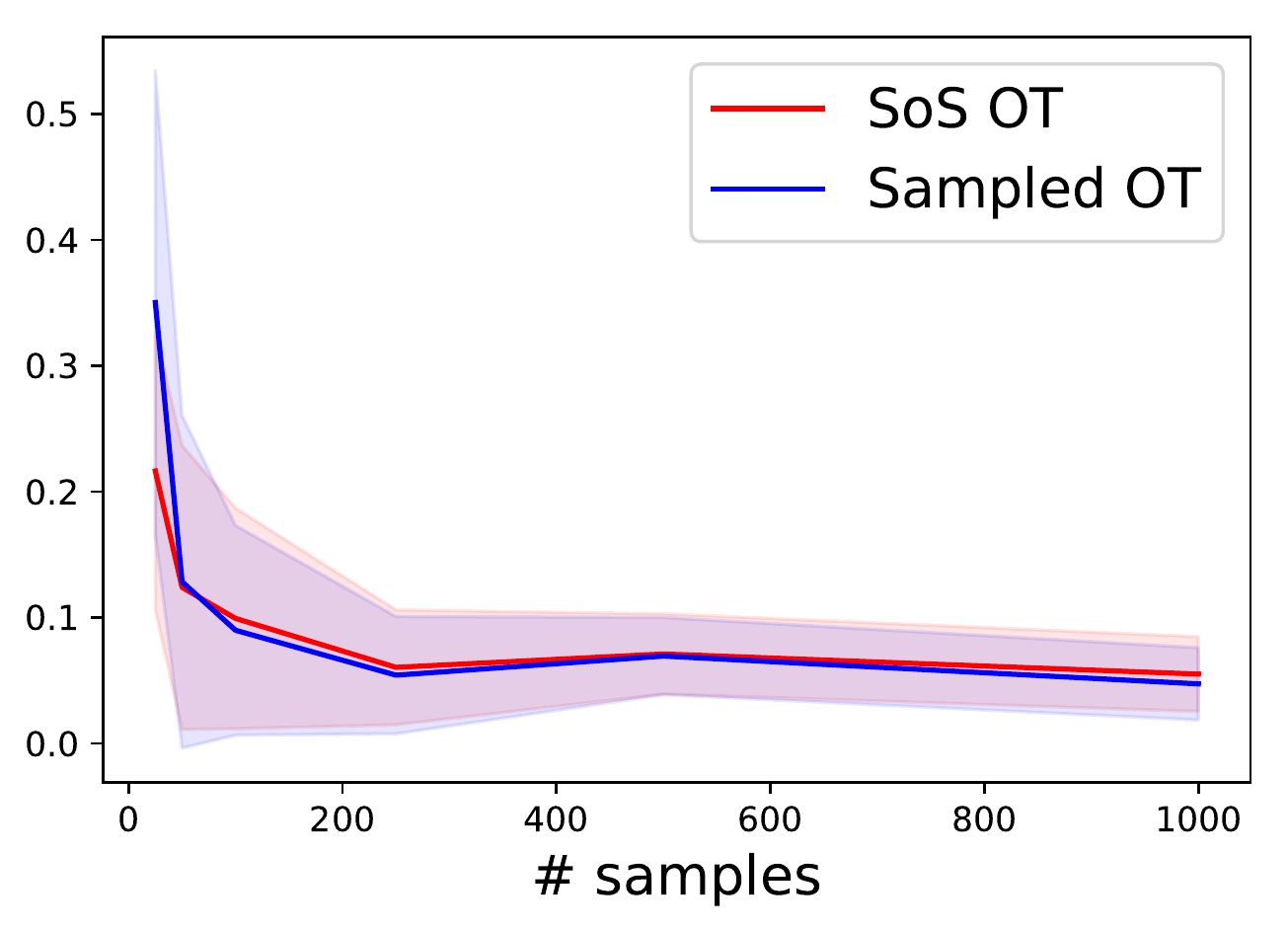}
     \end{subfigure}
     \hfill
         \begin{subfigure}[b]{0.45\textwidth}
         \centering
         \includegraphics[width=\textwidth]{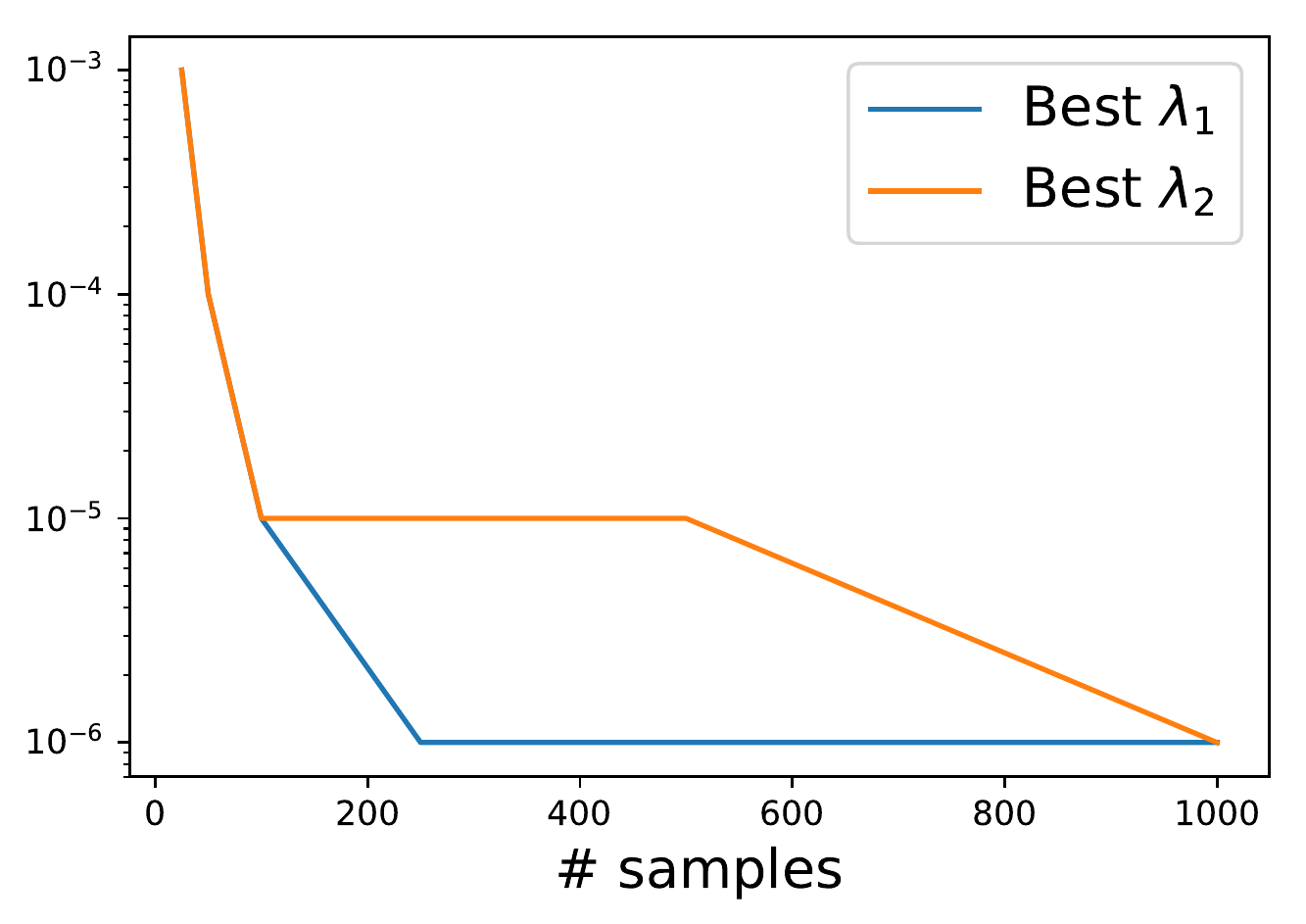}
     \end{subfigure}
     \begin{subfigure}[b]{0.45\textwidth}
         \centering
         \includegraphics[width=\textwidth]{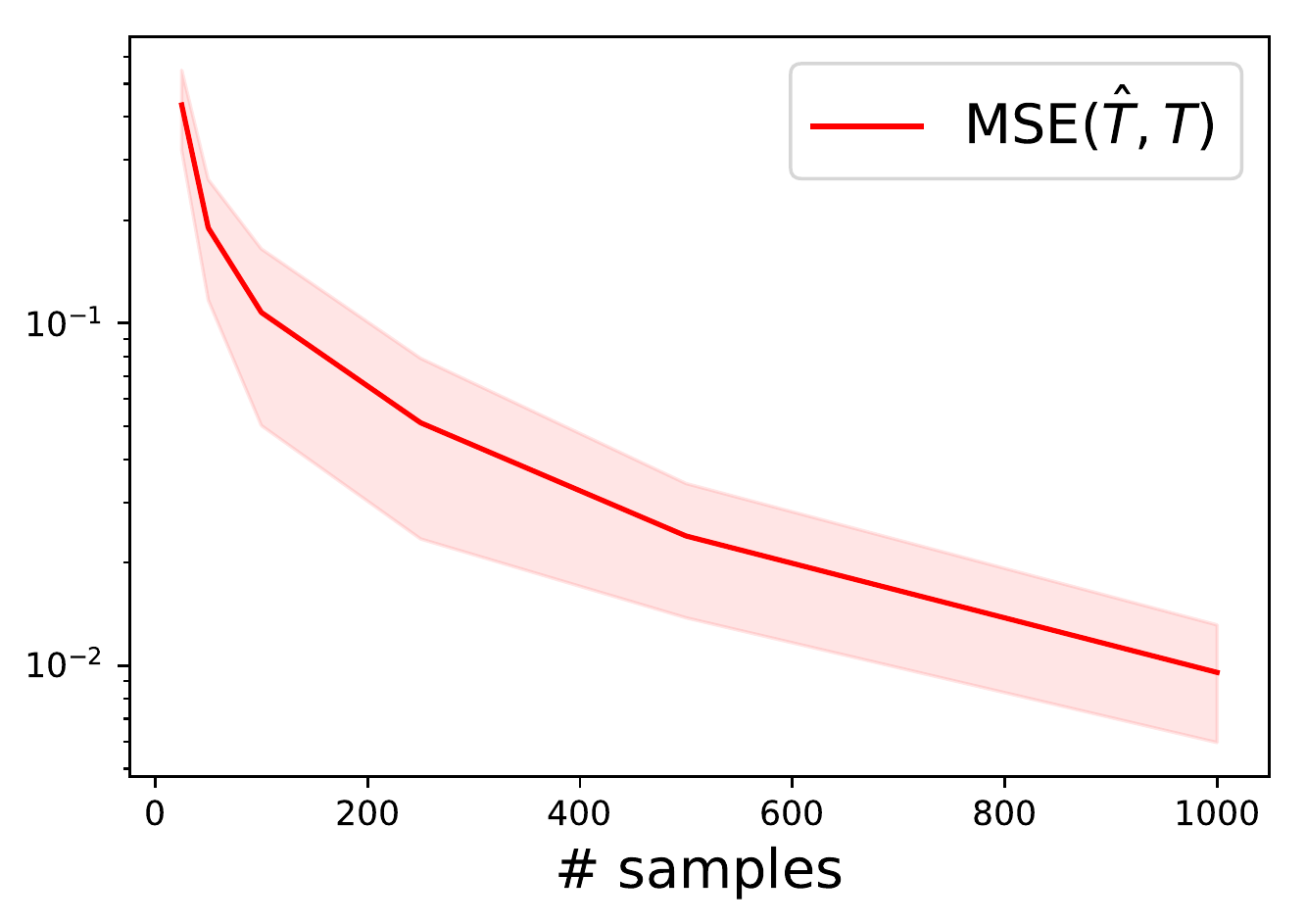}
     \end{subfigure}
     \hfill
         \begin{subfigure}[b]{0.45\textwidth}
         \centering
         \includegraphics[width=\textwidth]{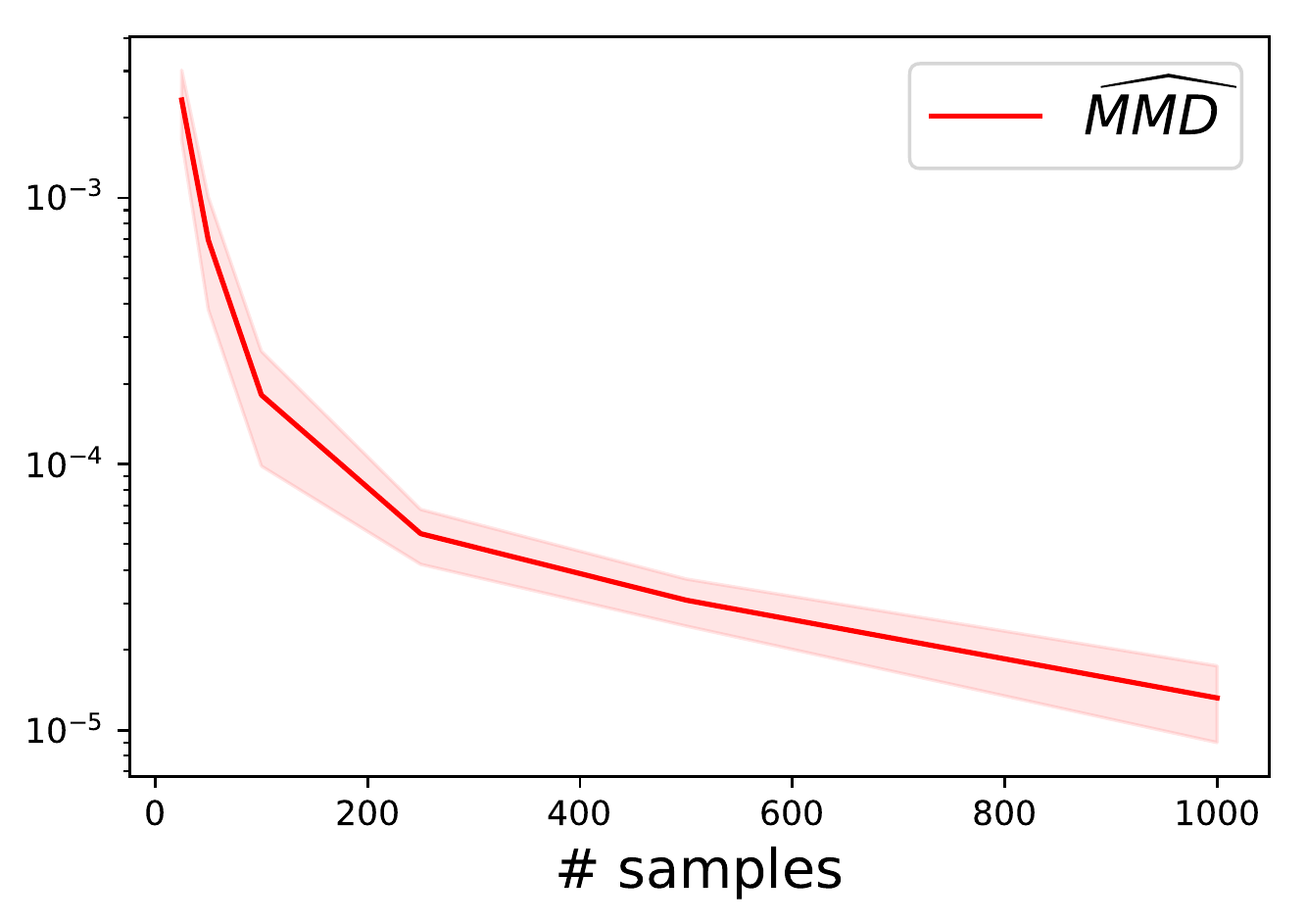}
     \end{subfigure}
    \caption{2D Gaussian data. {\em(Top left)} OT error $|\widehat{OT} - OT|$. {\em (Top right)} Best hyperparameters (selected via gridsearch), log scale. {\em (Bottom left)} Transportation map mean square error, log scale. {\em (Bottom right)} MMD between transported samples, log scale. Shaded areas correspond to $\pm$ std.  Algorithm: accelerated gradient descent ($\delta = 10^3$), Sobolev kernel ($s=20$, bandwidth $=1$), with Nystr\"om approximation (rank = $100$). Filling pairs $(\tx_i, \ty_i)$ are drawn from $\mu \otimes \nu$. The number of filling pairs is equal to the number of $\mu$ and $\nu$ samples, reported on the x-axis.}
    \label{fig:2D_gaussian}
\end{figure}

\begin{figure}
    \centering
    \begin{subfigure}[b]{0.45\textwidth}
         \centering
         \includegraphics[width=\textwidth]{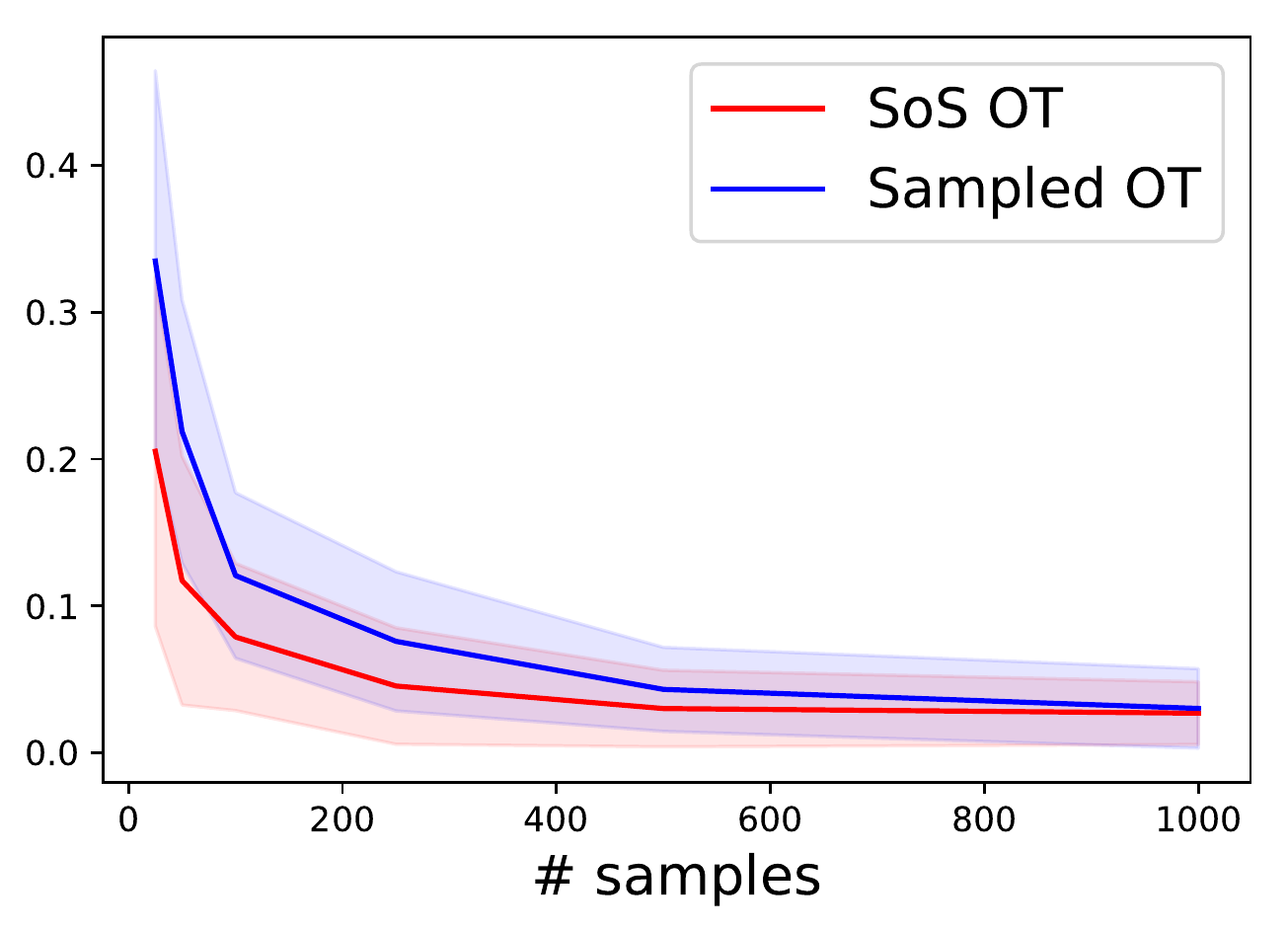}
     \end{subfigure}
     \hfill
         \begin{subfigure}[b]{0.45\textwidth}
         \centering
         \includegraphics[width=\textwidth]{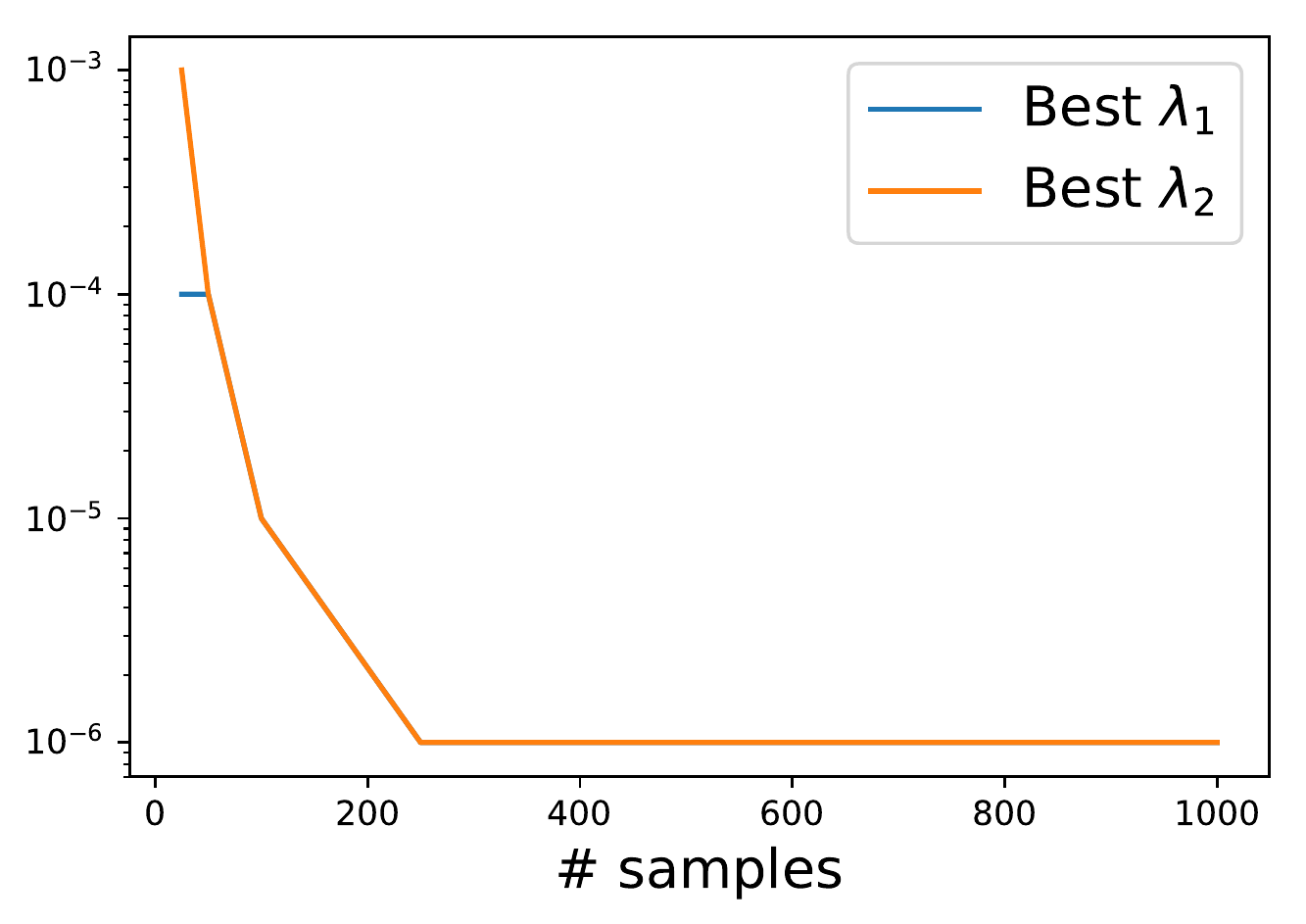}
     \end{subfigure}
     \begin{subfigure}[b]{0.45\textwidth}
         \centering
         \includegraphics[width=\textwidth]{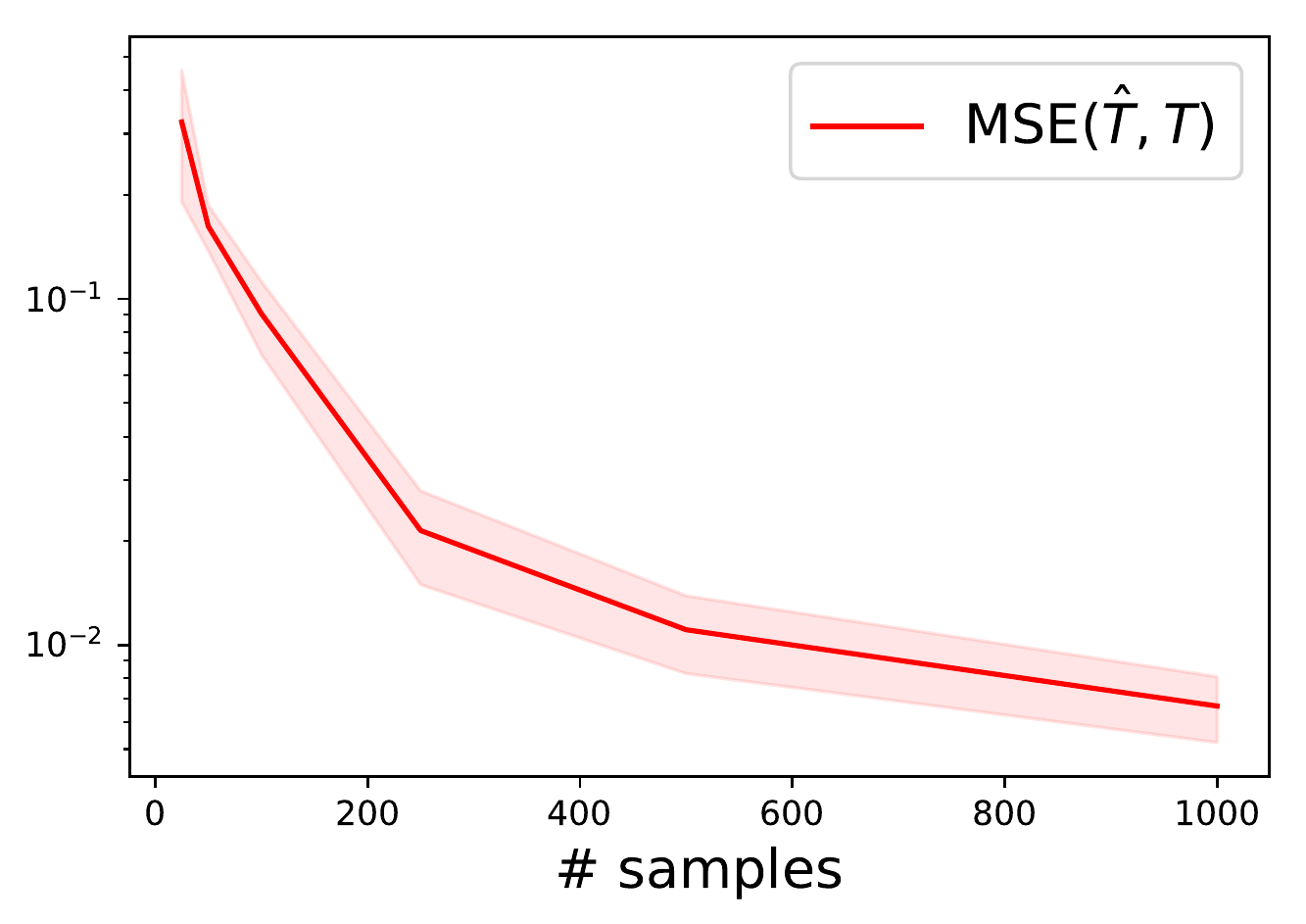}
     \end{subfigure}
     \hfill
         \begin{subfigure}[b]{0.45\textwidth}
         \centering
         \includegraphics[width=\textwidth]{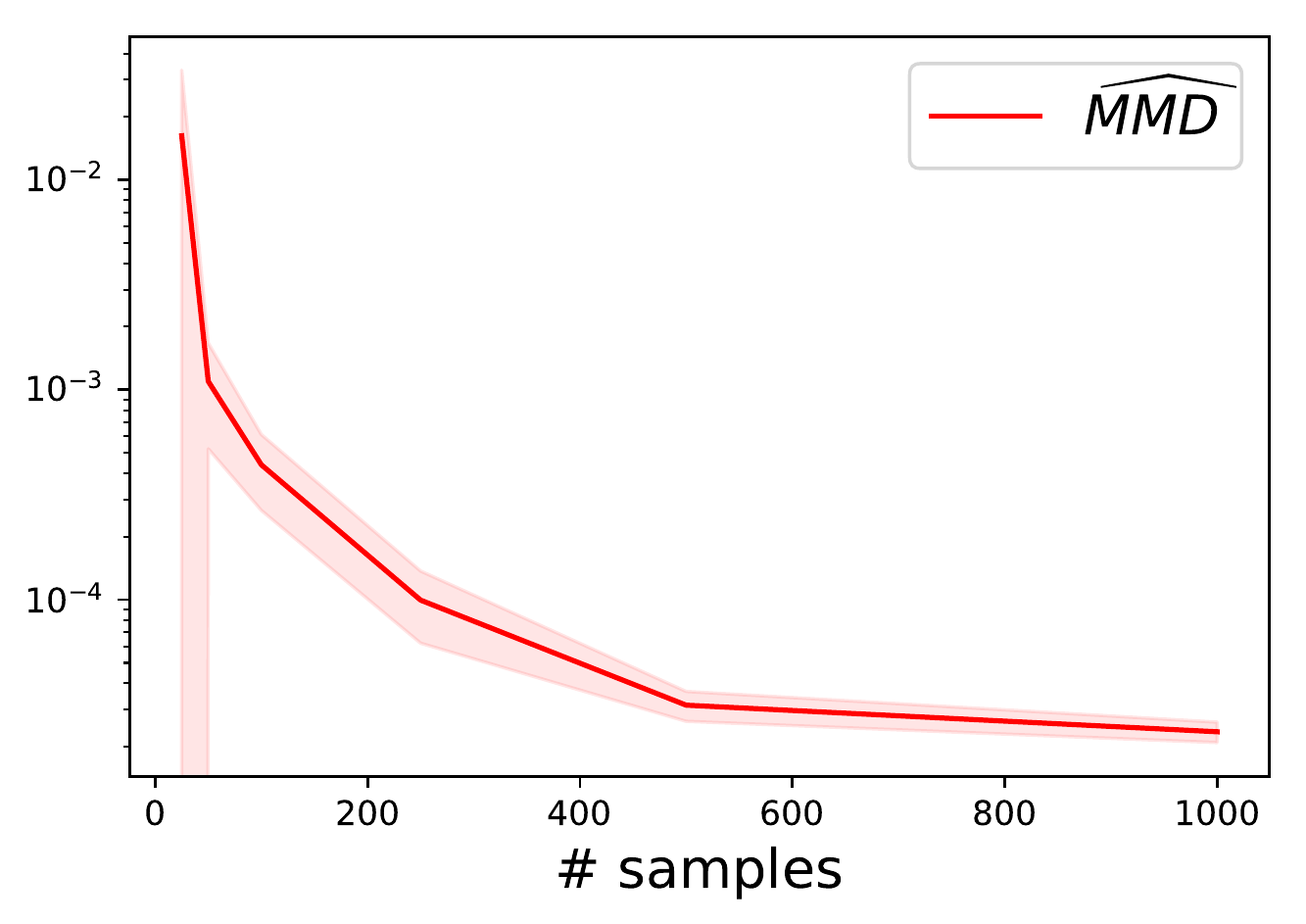}
     \end{subfigure}
    \caption{4D Gaussian data. {\em(Top left)} OT error $|\widehat{OT} - OT|$. {\em (Top right)} Best hyperparameters (selected via gridsearch), log scale. {\em (Bottom left)} Transportation map mean square error, log scale. {\em (Bottom right)} MMD between transported samples, log scale. Shaded areas correspond to $\pm$ std. Algorithm: accelerated gradient descent ($\delta = 10^3$), Sobolev kernel ($s=20$, bandwidth $=1$), with Nystr\"om approximation (rank = $100$). Filling pairs $(\tx_i, \ty_i)$ are drawn from $\mu \otimes \nu$. The number of filling pairs is equal to the number of $\mu$ and $\nu$ samples, reported on the x-axis.}
    \label{fig:4D_gaussian}
\end{figure}

\begin{figure}
    \centering
    \begin{subfigure}[b]{0.45\textwidth}
         \centering
         \includegraphics[width=\textwidth]{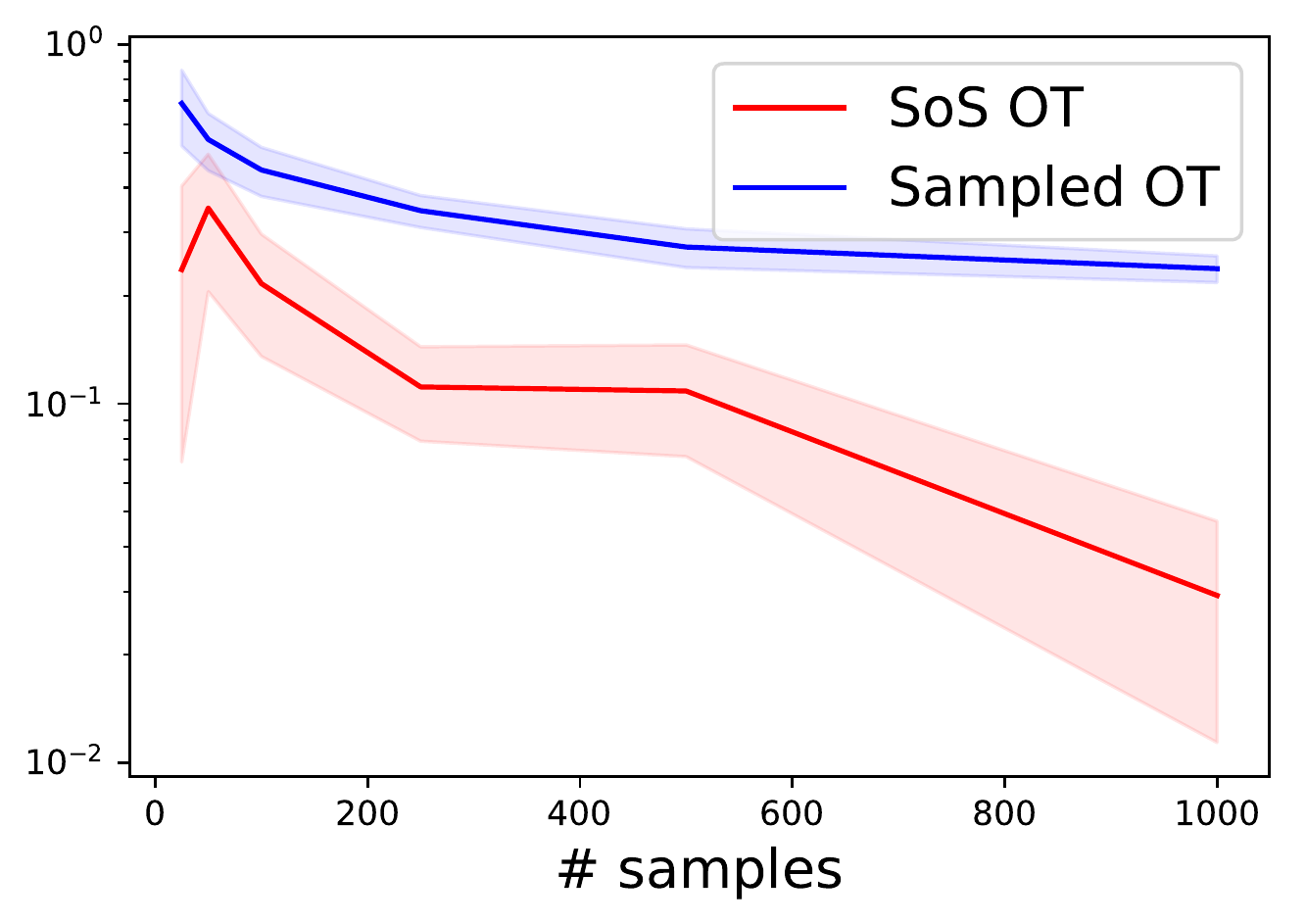}
     \end{subfigure}
     \hfill
         \begin{subfigure}[b]{0.45\textwidth}
         \centering
         \includegraphics[width=\textwidth]{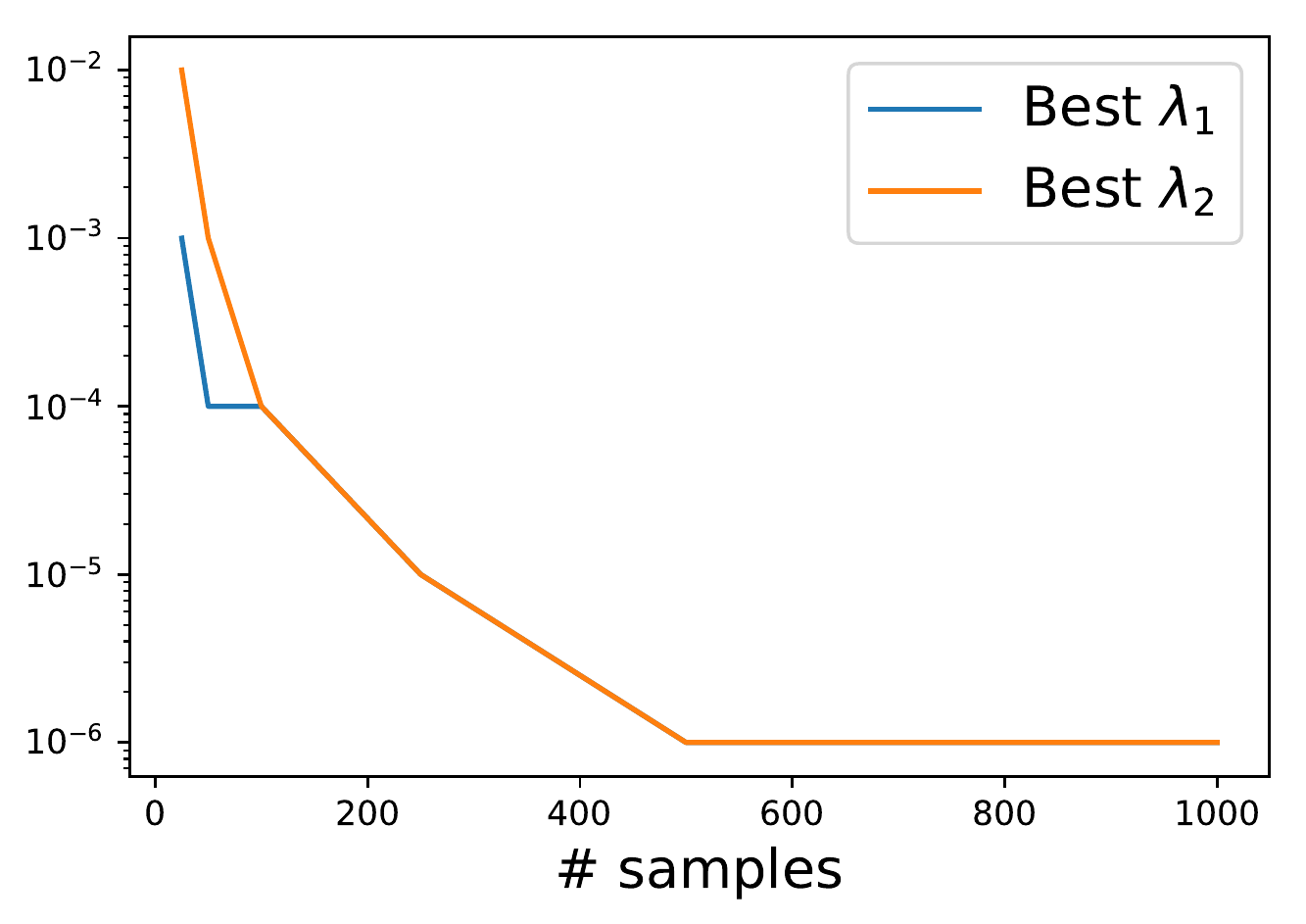}
     \end{subfigure}
     \begin{subfigure}[b]{0.45\textwidth}
         \centering
         \includegraphics[width=\textwidth]{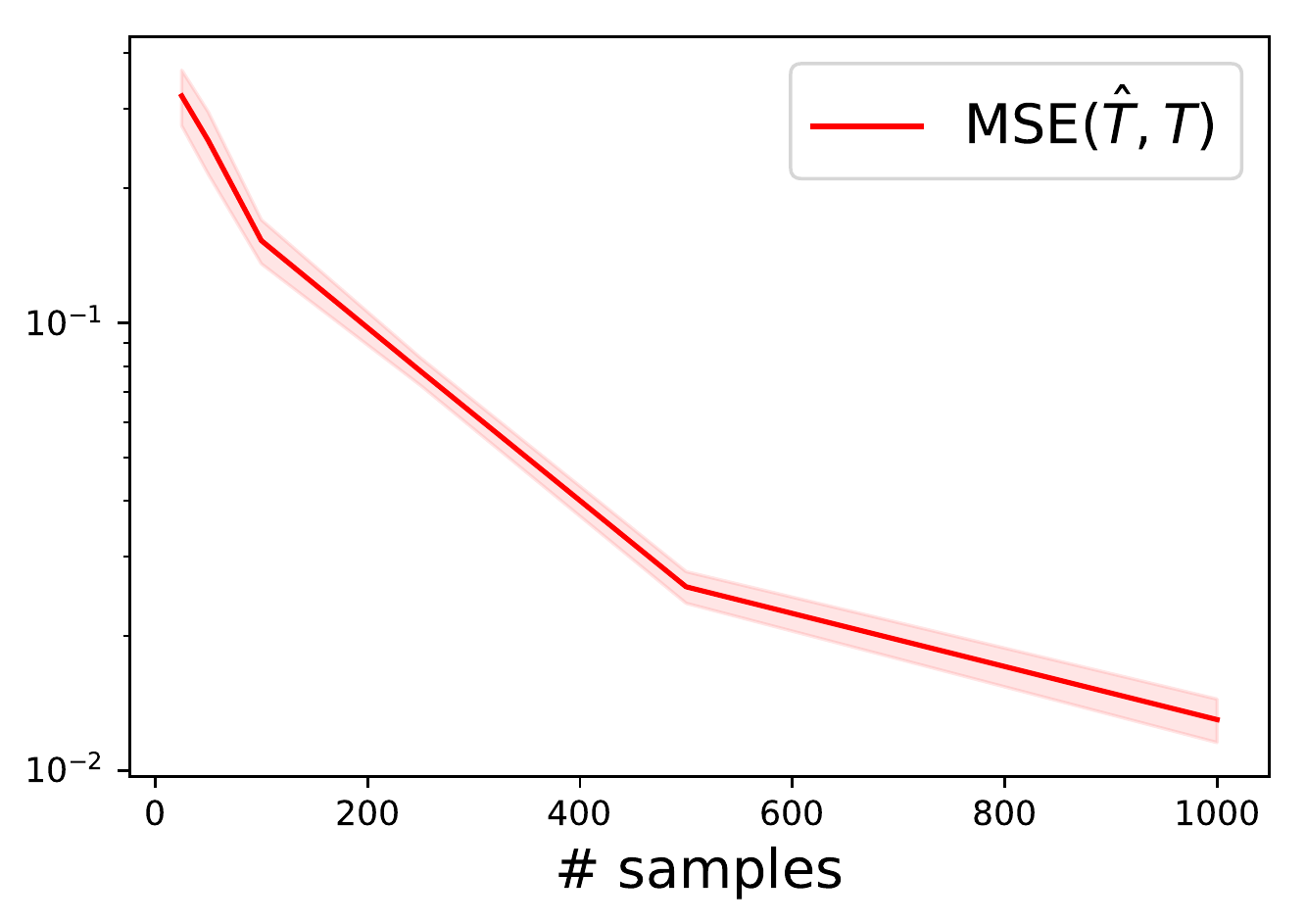}
     \end{subfigure}
     \hfill
         \begin{subfigure}[b]{0.45\textwidth}
         \centering
         \includegraphics[width=\textwidth]{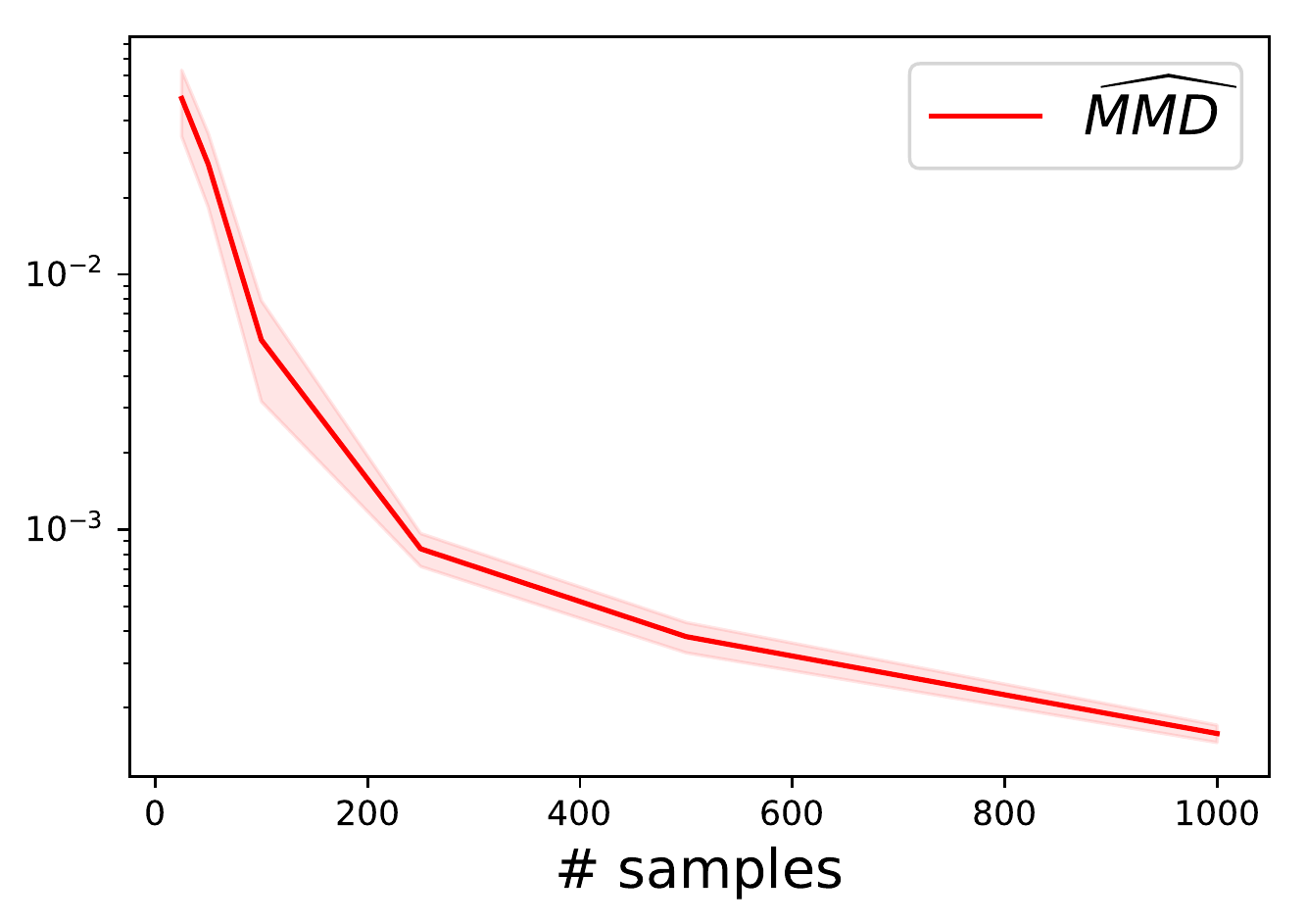}
     \end{subfigure}
    \caption{8D Gaussian data. {\em(Top left)} OT error $|\widehat{OT} - OT|$, log scale. {\em (Top right)} Best hyperparameters (selected via gridsearch), log scale. {\em (Bottom left)} Transportation map mean square error, log scale. {\em (Bottom right)} MMD between transported samples, log scale. Shaded areas correspond to $\pm$ std. Algorithm: accelerated gradient descent ($\delta = 10^3$), Sobolev kernel ($s=20$, bandwidth $=1$), with Nystr\"om approximation (rank = $100$). Filling pairs $(\tx_i, \ty_i)$ are drawn from $\mu \otimes \nu$. The number of filling pairs is equal to the number of $\mu$ and $\nu$ samples, reported on the x-axis. Note that the std in the top left is actually decreasing, but that the size of the shaded area visually increases due to logarithmic scaling.}
    \label{fig:8D_gaussian}
\end{figure}

We now report empirical evaluation of our estimators. In \Cref{fig:2D_gaussian,fig:4D_gaussian,fig:8D_gaussian}, we sample data from two Gaussian distributions (whose covariances follow a Wishart distribution), and solve problem \eqref{eq:w_hat_norm_pen_dual} for hyperparameters pairs $(\lambda_1, \lambda_2)$ on the $\{10^{-7}, 10^{-6}, 10^{-5}, 10^{-4}, 10^{-3}, 10^{-2}\}^2$ grid. As explained above, we then compute the forward and backward transportation maps \eqref{eq:forward_backward_map} corresponding to each $(\la_1, \la_2)$ pair, and report the performance of the pair minimizing \eqref{eq:mmd_criterion}, varying the number of samples from $25$ to $1000$, averaging over $20$ random draws. For each number of samples, we report the error of the OT distance estimator (like \cite{vacher2021dimensionfree}) compared to the plugin estimator,
the mean squared error of $\hat{T_1}$ and $\hat{T_2}$ compared to the ground truths $T_1$ and $T_2$ (which admit an analytical expression \citep[see][]{takatsu2011gaussian}) evaluated on the samples from $\mu\otimes \nu$: 
\begin{equation}\label{eq:MSE_criterion}
    \mathrm{MSE} = \frac{1}{n}\sum_{i=1}^n |\hat{T}_1(x_i) - T_1(x_i)|^2 + \frac{1}{n}\sum_{i=1}^n |\hat{T}_2(y_i) - T_2(y_i)|^2,
\end{equation}
the best hyperparameter pair and the MMD metric \eqref{eq:mmd_criterion}.
We may obtain several takeaways from \Cref{fig:2D_gaussian,fig:4D_gaussian,fig:8D_gaussian}. First, we see that gains of using the kernel SoS approach to OT increase with the dimension: while SoS-OT matches the performance of the sampled estimator in 2D (\cref{fig:2D_gaussian}), it becomes increasingly more efficient than the plugin estimator in 4D (\cref{fig:4D_gaussian}) and 8D (\cref{fig:8D_gaussian}). Second, we observe that the map estimator converges to the true map, and that the decrease of the $L_2$ error \eqref{eq:MSE_criterion} is well correlated with the MMD criterion \eqref{eq:mmd_criterion}. Finally, as expected, the optimal regularization parameters (selected from a gridsearch) decrease as the number of sample increases.

\subsection*{Acknowledgements}

This work was funded in part by the French government under management of Agence Nationale de la Recherche as part of the “Investissements d’avenir”
program, reference ANR-19-P3IA-0001 (PRAIRIE 3IA Institute). We also acknowledge
support from the European Research Council (grants SEQUOIA 724063 and REAL 947908),
and R\'egion Ile-de-France.

\bibliographystyle{plainnat}
\bibliography{biblio}


\appendix

\section{Additional Proofs}\label{AppendixAdditionalProofs}

\subsection{Proof of lemma \ref{lemma:j_strg_cvx}}
\begin{proof} First note that the Legendre transform is pointwise convex. For $(f,g)$ and $\lambda \in [0, 1]$, we have for all $y$
\begin{align}
    (\lambda f + (1 - \lambda)g)^{*}(y) & = \sup_{x} x^\top y - (\lambda f(x) + (1 - \lambda)g(x)) \\
    & \leq \lambda f^*(y) + (1 - \lambda)g^*(y) \, .
\end{align}
It follows that the semi-dual functional $J(f) = \langle f, \mu \rangle + \langle f^*, \nu \rangle$ is convex.
\par
Denoting $T_0$ the optimal transport map from $\mu$ to $\nu$, the semi-dual functional can be rewritten as 
$
    J(f) = \langle f ,\mu \rangle + \langle f^* \circ T_0,\mu \rangle\,.
$
Now, the Fenchel-Young inequality on $f$ gives for every couple $(x,y) \in X\times Y$
$
   y^\top x \leq f(x) + f^*(y) \,,
$
 Equality holds for $y = \nabla f(x)$, and in this case, one also have $x \in \partial f^*(\nabla f(x))$. To simplify notations, let us denote $T(x) = \nabla f(x)$. We get 
\begin{equation}\label{EqFenchelEquality}
    f(x) + f^*(T(x)) = T(x)^\top x \,.
\end{equation}
Denoting $f_0$ an optimal potential, the optimality condition applied to $f_0$ gives
$
    f_0(x) + f_0^*(T_0(x)) - T_0(x)^\top x = 0\,\,, \forall x \in \on{Supp}(\mu) 
$
and by integration
$
    J(f_0) = \int  T_0(x)^\top x \, \,\mathrm{d}\mu \,.
$
Therefore, we have
\begin{equation}\label{eq:Jf_minus_Jfstar}
    J(f) - J(f_0) = \int f(x) + f^*(T_0(x)) - T_0(x)^\top x\, \mathrm{d}\mu\,.
\end{equation}
Subtracting \cref{EqFenchelEquality} applied to $f$ pointwise from the integrand in \cref{eq:Jf_minus_Jfstar}, we get 
\begin{equation}
    J(f) - J(f_0) = \int f^*(T_0(x)) - f^*(T(x)) - (T_0(x) - T(x))^{\top} x \, \mathrm{d}\mu\,.
\end{equation}
Applying the case of equality in Fenchel-Young (see above), we have $x \in \partial f^*(T(x))$ where $\partial f^*$ is the sub-gradient of $f^*$. Hence, denoting $z_1 = T_0(x)$ and $z_2 = T(x)$, the integrand has the form
\begin{equation}
\label{eqFormIntegrand}
    f^*(z_1) - f^*(z_2) - (z_1 - z_2)^\top \partial f^*(z_2) \, .
\end{equation}
Since  $f$ is a $C^1$ convex function  with a $M$-Lipschitz gradient, then $f^*$ is $\frac 1M$-strongly convex and \eqref{eqFormIntegrand} is lower bounded as 
\begin{equation}
    f^*(z_1) - f^*(z_2) - (z_1 - z_2)^\top \partial f^*(z_2) \geq \frac{1}{2M} \| z_1 - z_2\|^2 \, ,
\end{equation}
and in particular, we recover
\begin{equation}
    J(f) - J(f_0) \geq \frac{1}{2M} \int \| T_0(x) - T(x) \|^2 \mathrm{d}\mu(x) \, .
\end{equation}
\end{proof}

\subsection{Proof of \cref{PropAdmissibleEmpiricalPotentials}}

\begin{proof} For $E \subset \mathbb{R}^d$ with lipschitz boundary and points $(e_i)_{1 \leq i \leq n}$, define the \textit{fill distance} $h(E, e)$ as
\begin{equation}
    h(E, e) = \sup_{z \in E} \min_{i \in [n]} \| z - e_i\|_2 \, .
\end{equation}
If $f$ has smoothness $m$, the sampling inequalities \citep{narcowich2005sobolev} state that if $\forall i \in [n]$, $f(e_i) = 0$ then there exists a positive constant $B$ depending on the smoothness $m$, the dimension $d$ and the geometry of $E$ such that if $h \leq h_0(m, d, E)$
\begin{equation}
    \sup_{e \in E} |f(e)| \leq B h^{m-d/2} \| f \|_{H^{m}(E)} \, .
\end{equation}
Now, recall that at the optimum, the function
\begin{equation*}
    \hat{\alpha}_n(x,y) := \hat{f}_n(x) + \hat{g}_n(y) - x^\top y - \langle \phi(x,y), \hat{A}_n \phi(x, y) \rangle_{H^{m}(X\times Y)}
\end{equation*}
is such that $\forall i \in [n] , ~ \hat{\alpha}_n(x_i,y_i) = 0$. Applying the previous result yields 
\begin{equation}
    \sup_{(x,y) \in X\times Y} | \hat{\alpha}_n(x,y) | \leq B h(X \times Y, (\hat{x}, \hat{y}))^{m-d} \| \hat{\alpha}_n \|_{H^{m}(X\times Y)} \, .
\end{equation}
Using Lemma 12 from \cite{vacher2021dimensionfree}, if $n \geq n_0(X, Y)$ there exists a universal constant $C_0$ such that we have with probability at least $1-\delta$ 
\begin{equation}
     h(X \times Y, (\hat{x}, \hat{y})) \leq C_0 \biggl( \frac{\log(\frac{n}{\delta})}{n} \biggr)^{\frac{1}{2d}} \, .
\end{equation}
There remains to upper bound $\| \hat{\alpha}_n \|_{H^{m}(X\times Y)}$
\begin{equation}
\begin{split}
    \| \hat{\alpha}_n \|_{H^{m}(X\times Y)} \leq & \| \hat{f}_n(\cdot) \|_{H^{m}(X \times Y)} + \| \hat{g}_n(\cdot) \|_{H^{m}(X \times Y)} + \|c \|_{H^{m}(X \times Y)} \\
    & + \|\langle \phi(\cdot, \cdot), \hat{A}_n \phi(\cdot, \cdot) \rangle_{H^{m}(X\times Y)} \|_{H^{m}(X \times Y)} \, ,
\end{split}
\end{equation}
where $c(x,y) = x^\top y$. The term $\| \hat{f}_n(\cdot) \|_{H^{m}(X \times Y)}$ is upper bounded by $C_1 \| \hat{f}_n(\cdot) \|_{H^{m+2}(X)}$ where $C_1$ is the embedding constant of $H^{m+2}(X)$ in $H^{m}(X)$. Using Lemma 9 in \citet{rudi2020global}, the term $\|\langle \phi(\cdot, \cdot), \hat{A}_n \phi(\cdot, \cdot) \rangle_{H^{m}(X\times Y)}\|_{H^{m}(X \times Y)}$ can be upper bounded by $C_2 \tr(\hat{A}_n)$ where $C_2$ is a constant depending on $X, Y, m, d$. Hence we obtain
\begin{equation}
     \| \hat{\alpha}_n \|_{H^{m}(X\times Y)} \leq C_1  \| \hat{f}_n \|_{H^{m+2}(X)} + C_3\| \hat{g}_n \|_{H^{m+2}(Y)} + C_2 \tr(\hat{A}_n) + G \, ,
\end{equation}
where $C_3$ is the embedding constant of $H^{m+2}(Y)$ in $H^{m}(Y)$ and $G =  \| c \|_{H^{m}(X \times Y)} < \infty $. In particular, since $\forall (x, y) \in X \times Y, \langle \phi(x,y), \hat{A}_n \phi(x, y) \rangle_{H^{m}(X\times Y)} \geq 0$, we obtain that for $n \geq n_0(X, Y, m, d)$, we have with probability at least $1 - \delta $
\begin{equation}
    \forall (x, y) \in X \times Y, ~\hat{f}_n(x) + \hat{g}_n(y) + C_0 C_4 \biggl( \frac{\log(\frac{n}{\delta})}{n} \biggr)^{\frac{m-d}{2d}}( \sqrt{2} \hat{R}_n + \tr(\hat{A}_n) + \frac{G}{C_4}) \geq x^\top y \, ,
\end{equation}
where $C_4 = \max (C_1, C_2, C_3)$ and $\hat{R}_n^2 = \| \hat{f}_n \|_{H^{m+2}(X)}^2 + \| \hat{g}_n \|_{H^{m+2}(Y)}^2$. 
\end{proof}

\subsection{Proof of \Cref{lemma:scal_ub}}

\begin{proof}
In order to apply the Poincaré inequality, we need to re-normalize the function $u-v$ such that it integrates to $0$. Denoting the residual $r = \int_X u(x) - v(x) dx$, we have that $\langle u - v, \mu - \hat{\mu} \rangle = \langle u - v - r, \mu - \hat{\mu} \rangle$ since $\mu$ and $\hat \mu$ have the same total mass. For a probability measure $\alpha$ and a RKHS $H$ with kernel $k$, we define kernel mean embedding $w_\alpha = \mathbb{E}_{\alpha}(k(X, .)) \in H$. With this notation, using the reproducing property, we can re write $\langle u - v - r, \mu - \hat{\mu} \rangle$ as $\langle u - v - r,  w_\mu - w_{\hat{\mu}} \rangle_{H^{d/2 + \varepsilon}(X)}$. Using Cauchy-Schwarz, we have the upper-bound
\begin{equation}
    \langle u - v, \mu - \hat{\mu} \rangle \leq \|u - v - r\|_{H^{d/2 + \varepsilon}(X)} \|w_\mu - w_{\hat{\mu}}\|_{H^{d/2 + \varepsilon}(X)} \, .
\end{equation}
We can apply the Pinelis inequality \citep[see Proposition 2 of][]{caponnetto2007optimal} to control the second term. We have with probability at least $1 - \delta$ 
\begin{equation}
    \|w_\mu - w_{\hat{\mu}}\|_{H^{d/2 + \varepsilon}(X)} \leq k_{d, \varepsilon} \frac{\log(\frac{2}{\delta})}{\sqrt{n}} \, ,
\end{equation}
where $k_{d, \varepsilon}$ is a constant that goes to infinity as $1/\varepsilon$ when $\varepsilon \to 0$. Now let us deal with the first term. There exists a constant $C$ such that  
\begin{equation}
    \|u - v - r\|_{H^{d/2 + \varepsilon}(X)} \leq C (\|\nabla u - \nabla v\|_{H^{d/2 + \varepsilon - 1}(X)} + \| u - v - r\|_{L^2(X)}) \, .
\end{equation}
For the second term, Poincaré-Wirtinger inequality \citep{meyers_poincare} states that there exists a constant $C'$ such that $\| u - v - r\|_{L^2(X)} \leq C'\|\nabla u - \nabla v\|_{L^2(X)}$. For the first term, we apply the Gagliardo-Nirenberg inequality which yields
\begin{equation}
    \|\nabla u - \nabla v\|_{H^{d/2 + \varepsilon - 1}(X)} \leq C_1''\|\nabla u - \nabla v\|_{H^{m+1}(X)}^{\frac{d/2 + \varepsilon - 1}{m+1}}\|\nabla u - \nabla v\|_{L^2(X)}^{1 - \frac{d/2 + \varepsilon - 1}{m+1}} + C_2''\|\nabla u - \nabla v\|_{L^2(X)} \, ,
\end{equation}
for $C_1'', C_2''$ constants independent of $u,v$ and $\delta$. Conversely, there exists $C'''$ such that
\begin{equation}
    \|\nabla u - \nabla v\|_{H^{m+1}(X)} \leq C'''\| u - v\|_{H^{m+2}(X)} \, .
\end{equation}
Finally, denoting $\mu_0 = \inf_{x \in X} \frac{d \mu(x)}{d \lambda(x)}$, we have $\|.\|_{L^2(X)} \leq \frac{\| . \|_{L^2(\mu)}}{\mu_0}$ and we obtain with probability $1-\delta$
\begin{equation*}
        \langle u - v, \mu - \hat{\mu} \rangle \leq K_0\frac{\log(\frac{2}{\delta})}{\sqrt{n}}\biggl( K_1\|u - v\|_{H^{m+2}(X)}^{\frac{d/2 + \varepsilon - 1}{m+1}}\| \nabla u - \nabla v \|_{L^2(\mu)}^{\frac{m+2-d/2-\varepsilon}{m+1}} + K_2 \|\nabla u - \nabla v\|_{L^2(\mu)} \biggr) \, ,
\end{equation*}
where $K_0 =  k_{d, \varepsilon}C$, $K_1 = \frac{(C''')^{\frac{d/2 + \varepsilon - 1}{m+1}}}{\mu_0^{\frac{m+2-d/2-\varepsilon}{m+1}}}$ and $K_2 =  \frac{C(C' + C_2'')}{\mu_0}$. Denoting $C_\mu = K_0(K_1 + K_2)$, we get
\begin{equation}
     \langle u - v, \mu - \hat{\mu} \rangle \leq C_\mu\frac{\log(\frac{2}{\delta})}{\sqrt{n}}\biggl(\|u - v\|_{H^{m+2}(X)}^{\frac{d/2 + \varepsilon - 1}{m+1}}\| \nabla u - \nabla v \|_{L^2(\mu)}^{\frac{m+2-d/2-\varepsilon}{m+1}} + \|\nabla u - \nabla v\|_{L^2(\mu)}\biggr) \, .
\end{equation}

\end{proof}

\subsection{Proof of \Cref{prop:grad_lip_ub}}

\begin{proof}
Recall that $\hat{L}_f = \| \tilde{f}_n(\hat{t}_f) \|_{W^{2, \infty}(X)}$ and that $ \tilde{f}_n(\hat{t}_f) = f_* + \hat{t}_f(\hat{f}_n - f_*)$. Hence $\hat{L}_f$ can be upper-bounded as
\begin{equation}
    \hat{L}_f \leq \|f_*\|_{W^{2, \infty}(X)} +  \hat{t}_f \|\hat{f}_n - f_*\|_{W^{2, \infty}(X)} \, ,
\end{equation}
and as a consequence $\frac{\hat{L}_f}{\hat{t}_f} \leq \frac{\|f_*\|_{W^{2, \infty}(X)}}{\hat{t}_f} + \|\hat{f}_n\|_{W^{2, \infty}(X)} + \|f_*\|_{W^{2, \infty}(X)}$. Now recall that $\hat{t}_f$ is given by
\begin{equation}
    \hat{t}_f = \min \biggl(1, \frac{\gamma }{2 \|\hat{f}_n - f_*\|_{W^{2, \infty}(X)}}\biggr) \, ,
\end{equation}
and in particular 
\begin{equation}
    \frac{1}{ \hat{t}_f} \leq \frac{2 \|\hat{f}_n - f_*\|_{W^{2, \infty}(X)}}{\gamma} + 1 \leq \frac{2 \biggl(\|\hat{f}_n \|_{W^{2, \infty}(X)} + \|f_*\|_{W^{2, \infty}(X)}\biggr)}{\gamma} + 1 \, ,
\end{equation}
which yields
\begin{align}
    \frac{\hat{L}_f}{\hat{t}_f} & \leq \|f_*\|_{W^{2, \infty}(X)}\biggl(\frac{2 \|f_*\|_{W^{2, \infty}(X)}}{\gamma} + 2\biggr) + \|\hat{f}_n\|_{W^{2, \infty}(X)}\biggl(\frac{2 \|f_*\|_{W^{2, \infty}(X)}}{\gamma} + 1\biggr) \\
    & \leq  2(\|f_*\|_{W^{2, \infty}(X)} + \|\hat{f}_n\|_{W^{2, \infty}(X)})\biggl(\frac{ \|f_*\|_{W^{2, \infty}(X)}}{\gamma} + 1\biggr) \\
    & \leq 2K_X (\|f_*\|_{H^{m+2}(X)} + \|\hat{f}_n\|_{H^{m+2}(X)})\biggl(\frac{ \|f_*\|_{W^{2, \infty}(X)}}{\gamma} + 1\biggr) \, ,
\end{align}
where $K_X$ is the embedding constant of $W^{2, \infty}(X)$ into $H^{m+2}(X)$ \citep{adams2003sobolev}.
\end{proof}

\subsection{Proof of \Cref{prop:asymptotic_behavior}}

\begin{proof} The objective of the proof is to find a parameter $\lambda_n$ as small as possible that keeps the quantities $\hat{R}_n$ and $\tr(\hat{A}_n)$ bounded ; the rate in $\lambda_n$ immediately follows the boundedness.
\par 
Denoting $b_n = \hat{R}_n + R$ and $c_n = \lambda_n(\tr(A_*) + R^2) + \hat{\kappa}_{n, \delta}$, we have the following system of inequalities
\begin{equation}\label{eq:cases_raw}
    \begin{cases}
a_n^2 \leq KC' b_n \biggl[ \frac{\log(\frac{2}{\delta})}{\sqrt{n}}\biggl(b_n^{\frac{d/2 + \varepsilon - 1}{m+1}} a_n^{\frac{m+2 - d/2 - \varepsilon}{m+1}} + a_n\biggr) + c_n \biggr] \\
 \lambda_n (\hat{R}_n^2 + \tr(\hat{A}_n)) \leq (2 C' + 1)\biggl[ \frac{\log(\frac{2}{\delta})}{\sqrt{n}}\biggl( b_n^{\frac{d/2 + \varepsilon - 1}{m+1}} a_n^{\frac{m+2 - d/2 - \varepsilon}{m+1}}+ a_n \biggr) +
    c_n \biggr]\,.
\end{cases}
\end{equation}

We want to split the analysis into two parts: the indexes for which 
\begin{equation}
    \begin{cases}
\frac{\log(\frac{2}{\delta})}{\sqrt{n}}( b_n^{\frac{d/2 + \varepsilon - 1}{m+1}} a_n^{\frac{m+2 - d/2 - \varepsilon}{m+1}}+ a_n ) > c_n \text{ (Case 1) } \\
\frac{\log(\frac{2}{\delta})}{\sqrt{n}}( b_n^{\frac{d/2 + \varepsilon - 1}{m+1}} a_n^{\frac{m+2 - d/2 - \varepsilon}{m+1}}+ a_n ) \leq c_n \text{ (Case 2) } \, .
\end{cases}
\end{equation}

Again, in the first case, we sub-split the analysis in two parts: the indexes for which
\begin{equation}
    \begin{cases}
a_n > b_n^{\frac{d/2 + \varepsilon - 1}{m+1}} a_n^{\frac{m+2 - d/2 - \varepsilon}{m+1}} \text{ (Case 1a) } \\
a_n \leq b_n^{\frac{d/2 + \varepsilon - 1}{m+1}} a_n^{\frac{m+2 - d/2 - \varepsilon}{m+1}} \text{ (Case 1b) } \, .
\end{cases}
\end{equation}

\paragraph{Case 1a.} On these indexes, we can re-write Equation \eqref{eq:cases_raw} as 
\begin{equation}
    \begin{cases}
a_n^2 \leq 4KC' b_n \frac{\log(\frac{2}{\delta})}{\sqrt{n}} a_n\\
 \lambda_n (\hat{R}_n^2 + \tr(\hat{A}_n)) \leq 4 (2 C' + 1) \frac{\log(\frac{2}{\delta})}{\sqrt{n}}a_n \, ,
\end{cases}
\end{equation}
Combining these two equations yields
\begin{equation}
    \lambda_n (\hat{R}_n^2 + \tr(\hat{A}_n)) \leq 16K(2 C' + 1)C' \frac{\log(\frac{2}{\delta})^2}{n} b_n \, .
\end{equation}
If we set $\lambda_n \geq \frac{\log(\frac{2}{\delta})}{n}$, we recover 
\begin{equation}
    \hat{R}_n^2 + \tr(\hat{A}_n) \leq  16K(2 C' + 1)C'(\hat{R}_n + R) \, ,
\end{equation}
which implies that $\hat{R}_n$ and $\tr(\hat{A}_n)$ are bounded independently on $\delta$ and that
\begin{equation}
    a_n^2 \leq (4KC' b_*)^2 \frac{\log(\frac{2}{\delta})^2}{n} \, ,
\end{equation}
where $b_* = \sup_n b_n$.

\paragraph{Case 1b.} On these indexes, the inequalities we obtain are 
\begin{equation}
    \begin{cases}
a_n^2 \leq 4 K C'\frac{\log(\frac{2}{\delta})}{\sqrt{n}} b_n^{\frac{m+d/2+\varepsilon}{m+1}} a_n^{\frac{m+2 - d/2 - \varepsilon}{m+1}} \\
 \lambda_n (\hat{R}_n^2 + \tr(\hat{A}_n)) \leq 4(2 C' + 1) \frac{\log(\frac{2}{\delta})}{\sqrt{n}} b_n^{\frac{d/2 + \varepsilon - 1}{m+1}} a_n^{\frac{m+2 - d/2 - \varepsilon}{m+1}} \, .
\end{cases}
\end{equation}
The first equation implies that $a_n^{\frac{m + d/2 + \varepsilon}{m+1}} \leq 4 KC'\frac{\log(\frac{2}{\delta})}{\sqrt{n}}  b_n^{\frac{m+d/2+\varepsilon}{m+1}}$ which also gives
\begin{equation}
    a_n^{\frac{m + 2 - d/2 - \varepsilon}{m+1}} \leq (4 KC')^{\frac{m + 2 - d/2 - \varepsilon}{m+d/2+\varepsilon}} b_n^{\frac{m + 2 - d/2 - \varepsilon}{m+1}}\biggl(\frac{\log(\frac{2}{\delta})^2}{n}\biggr)^{\frac{m + 2 - d/2 - \varepsilon}{2m+d+2\varepsilon}} \, .
\end{equation}
Hence the upper bound we obtain on $\lambda_n (\hat{R}_n^2 + \tr(\hat{A}_n))$ is 
\begin{align}
    \lambda_n (\hat{R}_n^2 + \tr(\hat{A}_n)) & \leq 4(2 C' + 1) (4 KC')^{\frac{m + 2 - d/2 - \varepsilon}{m+d/2+\varepsilon}} b_n \biggl(\frac{\log(\frac{2}{\delta})^2}{n}\biggr)^{\frac{1}{2}(1 + \frac{m + 2 - d/2 - \varepsilon}{m+d/2+\varepsilon})} \\
    & = 4(2 C' + 1) (4K C')^{\frac{m + 2 - d/2 - \varepsilon}{m+d/2+\varepsilon}} b_n \biggl(\frac{\log(\frac{2}{\delta})^2}{n}\biggr)^{\frac{m+1}{m+d/2+\varepsilon}} \, .
\end{align}
Choosing $\lambda_n \geq \biggl(\frac{\log(\frac{2}{\delta})^2}{n}\biggr)^{\frac{m+1}{m+d/2+\varepsilon}} $ gives that $\hat{R}_n$ and $\tr(\hat{A}_n)$ are bounded (independently on $\delta$) and yields the rate
\begin{equation}
    a_n^2 \leq (4 KC')^{\frac{m+1}{2m + d + 2 \varepsilon}} b_* \biggl(\frac{\log(\frac{2}{\delta})^2}{n}\biggr)^{\frac{m+1}{m+d/2+\varepsilon}} \, .
\end{equation}

\paragraph{Case 2.} On these indexes, the second equation of \eqref{eq:cases_raw} becomes 
\begin{equation}
     \lambda_n (\hat{R}_n^2 + \tr(\hat{A}_n)) \leq 2(2 C' + 1)c_n \, .
\end{equation}
Now recall that $c_n$ is given by
\begin{equation}
    c_n = \lambda_n (R^2 + \tr(A_*)) + C_1(\tr(\hat{A}_n) + \hat{R}_n + G)\biggl(\frac{\log(\frac{n}{\delta})}{n}\biggr)^{\frac{m-d}{2d}} \, ,
\end{equation}
where $C_1, G$ are constants that do not depend of $n, \delta$. Hence, if we choose $\lambda_n$ such that $\lambda_n \geq 4C_1(2 C' + 1) \biggl(\frac{\log(\frac{n}{\delta})}{n}\biggr)^{\frac{m-d}{2d}}$, the quantities $\hat{R}_n$ and $\tr(\hat{A}_n)$ are bounded independently on $\delta$ and we recover $a_n$ of the form
\begin{equation}
    a_n^2 \leq K' \biggl(\frac{\log(\frac{n}{\delta})}{n}\biggr)^{\frac{m-d}{2d}} \, ,
\end{equation}
where $K'$ is a constant independent on $n$ and $\delta$.

\paragraph{Conclusion.} If we set $\lambda_n =  \biggl(\frac{\log(\frac{2}{\delta})^2}{n}\biggr)^{\frac{m+1}{m+d/2+\varepsilon}}+ 4C_1(2 C' + 1) \biggl(\frac{\log(\frac{n}{\delta})}{n}\biggr)^{\frac{m-d}{2d}}$, we have that in any case, the energies are bounded and we get 
\begin{equation}
    a_n^2 \leq C \lambda_n \, ,
\end{equation}
where $C$ is a constant independent on $n$ and $\delta$.
\end{proof} 

\section{Proofs from \Cref{sec:algos}}\label{sec:algo_proofs}

\subsection{Proof of \Cref{lemma:frob_pen_dual}}\label{proof:frob_pen_dual}

\begin{proof}
 Not including the PSD constraints on $\bB$, problem \eqref{eq:EmpiricalDualBrenier_finite_dim} admits the following Lagrangian:
\begin{align}\label{eq:frob_lagrangian}
    \Lcal(f, g, \bB, \gamma) ~ =~& \scal{f}{\hat{w}_\mu}_\hhx + \scal{g}{\hat{w}_\nu}_\hhy + \frac{\la_1}{2} \|\bB\|_F^2 + \la_2(\|f\|^2_\hhx + \|g\|^2_\hhy) \\ 
    & + \sum_{j=1}^\ell \gamma_j (\Phi_j^T\bB\Phi_j - f(\tilde{x}_j) - g(\tilde{y}_j)  + \dotp{\tilde{x}_j}{\tilde{y}_j}).
\end{align}
 Canceling the gradients in $f$ and $g$, we get 
\begin{align}\label{eq:fg_opt_proof}
    \begin{split}
    f &= \frac{1}{2\lambda_2}(\sum_{i=1}^\ell \gamma_i \phi_X(\tilde{x}_i) - \hat{w}_\mu) \\
    g &= \frac{1}{2\lambda_2}(\sum_{i=1}^\ell \gamma_i \phi_Y(\tilde{y}_i) - \hat{w}_\nu).
    \end{split}
\end{align}
Let us now derive the optimality condition on $\bB$: completing the square, we have
\begin{align}\label{eq:B_opt_proof}
\begin{split}
    \underset{\bB\in\SS_+(\RR^\ell)}{\inf} \sum_{i=1}^\ell\gamma_{i}\Phi_i^T\bB\Phi_i + \frac{\la_1}{2} \| \bB\|_F^2
    &= \underset{\bB\in\SS_+(\RR^\ell)}{\inf} \frac{\la_1}{2} \dotp{\bB}{\bB + \frac{2}{\lambda_1}\sum_{i=1}^\ell\gamma_{i}\Phi_i\Phi_i^T}\\
    & = -\frac{1}{2\la_1} \| (- \sum_{i=1}^\ell \gamma_i \Phi_i \Phi_i^T)_+\|_F^2.
\end{split}
\end{align}
Plugging \cref{eq:fg_opt_proof} and \cref{eq:B_opt_proof} into \cref{eq:frob_lagrangian}, we obtain \cref{eq:w_hat_norm_dual}.
\end{proof}

\subsection{Proof of \Cref{prop:relaxed_dual}}\label{proof:relaxed_dual}

\begin{proof}
\textbf{Dual formulation.} It holds
\begin{align}\label{eq:primal-dual}
\begin{split}
    \inf_{\substack{f \in \hhx, g\in \hhy, \\  \bB \in \pdm{\RR^d}}} ~~~ & \scal{f}{\hat{w}_\mu}_\hhx + \scal{g}{\hat{w}_\nu}_\hhy + \frac{\la_1}{2} \|\bB\|_F^2 + \la_2(\|f\|^2_\hhx + \|g\|^2_\hhy) \\ 
    &+ \frac{\delta}{2\ell} \sum_{j=1}^\ell  (f(\tilde{x}_j) + g(\tilde{y}_j) - \tilde{x}_j \cdot \tilde{y}_j -  \Phi_j^T\bB\Phi_j)^2,\\
    = \sup_{\gamma \in \RR^\ell} \inf_{\substack{f \in \hhx, g\in \hhy, \\  \bB \in \pdm{\RR^d}}} ~~~ & \scal{f}{\hat{w}_\mu}_\hhx + \scal{g}{\hat{w}_\nu}_\hhy + \frac{\la_1}{2} \|\bB\|_F^2 + \la_2(\|f\|^2_\hhx + \|g\|^2_\hhy) \\ 
    & + \sum_{j=1}^\ell \gamma_j (\Phi_j^T\bB\Phi_j - f(\tilde{x}_j) - g(\tilde{y}_j)  + \tilde{x}_j \cdot \tilde{y}_j) - \frac{\ell}{2\delta}\| \gamma\|^2.
\end{split}
\end{align}
Indeed, the second problem above is a convex-concave min-max problem, and inverting min and max and solving for $\gamma$ directly yields the original problem. Let us rewrite the inner $\inf$ as a function of $\gamma$.  
Canceling the gradients in $f$ and $g$, we get 
\begin{align}
    \begin{split}
    f &= \frac{1}{2\lambda_2}(\sum_{i=1}^\ell \gamma_i \phi_X(\tilde{x}_i) - \hat{w}_\mu) \\
    g &= \frac{1}{2\lambda_2}(\sum_{i=1}^\ell \gamma_i \phi_Y(\tilde{y}_i) - \hat{w}_\nu),
    \end{split}
\end{align}
i.e. the primal-dual relations \cref{eq:fg_opt}. Let us now derive the optimality condition on $\bB$: as in the proof of of \Cref{lemma:frob_pen_dual}, we have
\begin{align}\label{eq:B_opt}
\begin{split}
    \underset{\bB\in\SS_+(\RR^\ell)}{\inf} \sum_{i=1}^\ell\gamma_{i}\Phi_i^T\bB\Phi_i + \frac{\la_1}{2} \| \bB\|_F^2
    & = -\frac{1}{2\la_1} \| (- \sum_{i=1}^\ell \gamma_i \Phi_i \Phi_i^T)_+\|_F^2.
\end{split}
\end{align}
Plugging \cref{eq:fg_opt} and \cref{eq:B_opt} into \cref{eq:primal-dual}, we obtain \cref{eq:w_hat_norm_pen_dual}.
Let us now derive the smoothness and strong convexity constants of
\begin{equation}
    H(\gamma) \defeq \frac{1}{4\lambda_2} \gamma^T {\bf Q} \gamma - \frac{1}{2\lambda_2}\sum_{j=1}^\ell \gamma_j z_j + \frac{1}{2\la_1} \| (- \sum_{i=1}^\ell \gamma_i \Phi_i \Phi_i^T)_+\|_F^2 + \frac{\ell}{2\delta} \|\gamma\|^2 + \frac{q^2}{4\la_2}.
\end{equation}
\textbf{Strong convexity.}
Let us recall that $H$ is $\alpha$-strongly convex i.f.f.\ $\forall \gamma, \gamma', \alpha \|\gamma - \gamma'\|^2 \leq \nabla \dotp{H(\gamma) - \nabla H(\gamma')}{\gamma - \gamma'}$. We have 
\begin{align*}
&\dotp{H(\gamma) - \nabla H(\gamma')}{\gamma - \gamma'} = \frac{l}{\delta} \|\gamma - \gamma'\|^2 + \frac{1}{2\la_2}(\gamma - \gamma')^T\bQ(\gamma - \gamma')
\\&~~~~ + \frac{1}{\la_2}\diag(\Phi^T([\sum_{i=1}^\ell \gamma_i \Phi_i \Phi_i^T]_- - [\sum_{i=1}^\ell \gamma'_i \Phi_i \Phi_i^T]_-)\Phi)^T(\gamma - \gamma').
\end{align*}
Given that the term with negative parts is non-negative but vanishes when the corresponding matrices are positive, we get the lower bound $\alpha \geq \frac{\ell}{\delta} + \frac{1}{2 \la_2}\lambda_{\min}(\bQ)$.\\
\textbf{Smoothness.} Finally, we have 
\begin{align*}
\|H(\gamma) - \nabla H(\gamma')\| &\leq \|\frac{l}{\delta}(\gamma - \gamma')\| + \frac{1}{2\la_2}\|\bQ(\gamma - \gamma')\|
\\&~~~+ \frac{1}{\la_2}\| \diag(\Phi^T([\sum_{i=1}^\ell \gamma_i \Phi_i \Phi_i^T]_- - [\sum_{i=1}^\ell \gamma'_i \Phi_i \Phi_i^T]_-)\Phi)\|.
\end{align*}
Bounded those three terms independently yields 
$$
\|H(\gamma) - \nabla H(\gamma')\| \leq \left(\frac{\ell}{\delta} + \frac{1}{2\la_2} \la_{\max}(\bQ) + \frac{1}{\la_1} \la_{\max}(\bK \circ \bK)\right)\|\gamma - \gamma'\|.  
$$
\end{proof}

\end{document}

%% file: arxiv_macros.tex
\newtheorem{ass}{Assumption}
\newcommand{\ba}{\begin{ass}}
\newcommand{\ea}{\end{ass}}
\newtheorem{theorem}{Theorem}

\newtheorem{remark}[theorem]{Remark}

\newtheorem{proposition}[theorem]{Proposition}

\newtheorem{lemma}[theorem]{Lemma}

\newcommand{\jmlrBlackBox}{\rule{1.5ex}{1.5ex}}
 
\newcommand{\jmlrQED}{\hfill\jmlrBlackBox\par\bigskip}
  \renewenvironment{proof}%
  {%
   \par\noindent{\bfseries\upshape Proof\ }%
  }%
  {\jmlrQED}



%% file: arxiv.bbl
\begin{thebibliography}{38}
\providecommand{\natexlab}[1]{#1}
\providecommand{\url}[1]{\texttt{#1}}
\expandafter\ifx\csname urlstyle\endcsname\relax
  \providecommand{\doi}[1]{doi: #1}\else
  \providecommand{\doi}{doi: \begingroup \urlstyle{rm}\Url}\fi

\bibitem[Adams and Fournier(2003)]{adams2003sobolev}
Robert~A. Adams and John J.~F. Fournier.
\newblock \emph{Sobolev Spaces}.
\newblock Elsevier, 2003.

\bibitem[Ahuja et~al.(1993)Ahuja, Orlin, and Magnanti]{ahyja1993network}
Ravindra~K Ahuja, James~B Orlin, and Thomas~L Magnanti.
\newblock \emph{Network flows: theory, algorithms, and applications}.
\newblock Prentice-Hall, 1993.

\bibitem[Arjovsky et~al.(2017)Arjovsky, Chintala, and
  Bottou]{arjovsky2017wasserstein}
Martin Arjovsky, Soumith Chintala, and L{\'e}on Bottou.
\newblock Wasserstein generative adversarial networks.
\newblock In \emph{International conference on machine learning}, pages
  214--223. PMLR, 2017.

\bibitem[Bernton et~al.(2017)Bernton, Jacob, Gerber, and
  Robert]{bernton2017inference}
Espen Bernton, Pierre~E Jacob, Mathieu Gerber, and Christian~P Robert.
\newblock Inference in generative models using the wasserstein distance.
\newblock \emph{arXiv preprint arXiv:1701.05146}, 1\penalty0 (8):\penalty0 9,
  2017.

\bibitem[Brenier(1987)]{brenier1987decomposition}
Yann Brenier.
\newblock D{\'e}composition polaire et r{\'e}arrangement monotone des champs de
  vecteurs.
\newblock \emph{CR Acad. Sci. Paris S{\'e}r. I Math.}, 305:\penalty0 805--808,
  1987.

\bibitem[Caponnetto and De~Vito(2007)]{caponnetto2007optimal}
Andrea Caponnetto and Ernesto De~Vito.
\newblock Optimal rates for the regularized least-squares algorithm.
\newblock \emph{Foundations of Computational Mathematics}, 7\penalty0
  (3):\penalty0 331--368, 2007.

\bibitem[Chizat et~al.(2020)Chizat, Roussillon, L{\'e}ger, Vialard, and
  Peyr{\'e}]{chizat2020faster}
Lenaic Chizat, Pierre Roussillon, Flavien L{\'e}ger, Fran{\c{c}}ois-Xavier
  Vialard, and Gabriel Peyr{\'e}.
\newblock Faster {W}asserstein distance estimation with the {S}inkhorn
  divergence.
\newblock \emph{Advances in Neural Information Processing Systems}, 33, 2020.

\bibitem[Courty et~al.(2016)Courty, Flamary, Tuia, and
  Rakotomamonjy]{courty2016optimal}
Nicolas Courty, R{\'e}mi Flamary, Devis Tuia, and Alain Rakotomamonjy.
\newblock Optimal transport for domain adaptation.
\newblock \emph{IEEE transactions on pattern analysis and machine
  intelligence}, 39\penalty0 (9):\penalty0 1853--1865, 2016.

\bibitem[Courty et~al.(2017)Courty, Flamary, Habrard, and
  Rakotomamonjy]{courty2017joint}
Nicolas Courty, R{\'e}mi Flamary, Amaury Habrard, and Alain Rakotomamonjy.
\newblock Joint distribution optimal transportation for domain adaptation.
\newblock \emph{Advances in Neural Information Processing Systems}, pages
  3733--3742, 2017.

\bibitem[Cuturi(2013)]{cuturi2013sinkhorn}
Marco Cuturi.
\newblock Sinkhorn distances: Lightspeed computation of optimal transport.
\newblock \emph{Advances in neural information processing systems},
  26:\penalty0 2292--2300, 2013.

\bibitem[De~Philippis and Figalli(2014)]{philippis2013mongeampre}
Guido De~Philippis and Alessio Figalli.
\newblock The {M}onge--{A}mp{\`e}re equation and its link to optimal
  transportation.
\newblock \emph{Bulletin of the American Mathematical Society}, 51\penalty0
  (4):\penalty0 527--580, 2014.

\bibitem[Deb et~al.(2021)Deb, Ghosal, and Sen]{deb2021rates}
Nabarun Deb, Promit Ghosal, and Bodhisattva Sen.
\newblock Rates of estimation of optimal transport maps using plug-in
  estimators via barycentric projections.
\newblock \emph{Advances in Neural Information Processing Systems}, 34, 2021.

\bibitem[Feydy et~al.(2017)Feydy, Charlier, Vialard, and Peyr{\'e}]{Feydy2017}
Jean Feydy, Benjamin Charlier, Fran{\c{c}}ois-Xavier Vialard, and Gabriel
  Peyr{\'e}.
\newblock Optimal transport for diffeomorphic registration.
\newblock In \emph{International Conference on Medical Image Computing and
  Computer-Assisted Intervention}, pages 291--299. Springer, 2017.

\bibitem[Genevay et~al.(2016)Genevay, Cuturi, Peyr{\'e}, and
  Bach]{genevay2016stochastic}
Aude Genevay, Marco Cuturi, Gabriel Peyr{\'e}, and Francis Bach.
\newblock {Stochastic optimization for large-scale optimal transport}.
\newblock In \emph{{Advances in Neural Information Processing System}},
  volume~30, 2016.

\bibitem[Gretton et~al.(2012)Gretton, Borgwardt, Rasch, Sch{\"o}lkopf, and
  Smola]{gretton2012kernel}
Arthur Gretton, Karsten~M Borgwardt, Malte~J Rasch, Bernhard Sch{\"o}lkopf, and
  Alexander Smola.
\newblock A kernel two-sample test.
\newblock \emph{The Journal of Machine Learning Research}, 13\penalty0
  (1):\penalty0 723--773, 2012.

\bibitem[Gunsilius(2018)]{gunsilius_2021}
Florian~F Gunsilius.
\newblock On the convergence rate of potentials of {B}renier maps.
\newblock \emph{Econometric Theory}, pages 1--37, 2018.

\bibitem[H{\"u}tter and Rigollet(2021)]{hutter2019minimax}
Jan-Christian H{\"u}tter and Philippe Rigollet.
\newblock Minimax estimation of smooth optimal transport maps.
\newblock \emph{The Annals of Statistics}, 49\penalty0 (2):\penalty0
  1166--1194, 2021.

\bibitem[Makkuva et~al.(2020)Makkuva, Taghvaei, Oh, and
  Lee]{makkuva2020optimal}
Ashok Makkuva, Amirhossein Taghvaei, Sewoong Oh, and Jason Lee.
\newblock Optimal transport mapping via input convex neural networks.
\newblock In \emph{International Conference on Machine Learning}, pages
  6672--6681. PMLR, 2020.

\bibitem[Manole et~al.(2021)Manole, Balakrishnan, Niles-Weed, and
  Wasserman]{manole2021plugin}
Tudor Manole, Sivaraman Balakrishnan, Jonathan Niles-Weed, and Larry Wasserman.
\newblock Plugin estimation of smooth optimal transport maps.
\newblock \emph{arXiv preprint arXiv:2107.12364}, 2021.

\bibitem[Marteau-Ferey et~al.(2020)Marteau-Ferey, Bach, and
  Rudi]{marteau2020non}
Ulysse Marteau-Ferey, Francis Bach, and Alessandro Rudi.
\newblock Non-parametric models for non-negative functions.
\newblock \emph{Advances in Neural Information Processing Systems}, 2020.

\bibitem[Meyers and Ziemer(1977)]{meyers_poincare}
Norman~G. Meyers and William~P. Ziemer.
\newblock {Integral inequalities of Poincaré and Wirtinger type for BV
  functions}.
\newblock \emph{American Journal of Mathematics}, 99, 1977.

\bibitem[Narcowich et~al.(2005)Narcowich, Ward, and
  Wendland]{narcowich2005sobolev}
Francis Narcowich, Joseph Ward, and Holger Wendland.
\newblock Sobolev bounds on functions with scattered zeros, with applications
  to radial basis function surface fitting.
\newblock \emph{Mathematics of Computation}, 74\penalty0 (250):\penalty0
  743--763, 2005.

\bibitem[Onken et~al.(2021)Onken, Wu~Fung, Li, and Ruthotto]{onken2021ot}
D~Onken, S~Wu~Fung, Xingjian Li, and L~Ruthotto.
\newblock Ot-flow: Fast and accurate continuous normalizing flows via optimal
  transport.
\newblock In \emph{AAAI Conference on Artificial Intelligence}, volume~35,
  2021.

\bibitem[Paty et~al.(2020)Paty, d’Aspremont, and Cuturi]{paty2020convexity}
Fran{\c{c}}ois-Pierre Paty, Alexandre d’Aspremont, and Marco Cuturi.
\newblock Regularity as regularization: Smooth and strongly convex brenier
  potentials in optimal transport.
\newblock In \emph{International Conference on Artificial Intelligence and
  Statistics}, pages 1222--1232. PMLR, 2020.

\bibitem[Paulsen and Raghupathi(2016)]{paulsen2016introduction}
Vern~I. Paulsen and Mrinal Raghupathi.
\newblock \emph{{An Introduction to the Theory of Reproducing Kernel Hilbert
  Spaces}}, volume 152.
\newblock Cambridge University Press, 2016.

\bibitem[Pooladian and Niles-Weed(2021)]{pooladian2021entropic}
Aram-Alexandre Pooladian and Jonathan Niles-Weed.
\newblock Entropic estimation of optimal transport maps.
\newblock \emph{arXiv preprint arXiv:2109.12004}, 2021.

\bibitem[Rudi et~al.(2015)Rudi, Camoriano, and Rosasco]{rudi2015less}
Alessandro Rudi, Raffaello Camoriano, and Lorenzo Rosasco.
\newblock Less is more: Nystr{\"o}m computational regularization.
\newblock \emph{Advances in Neural Information Processing Systems},
  28:\penalty0 1657--1665, 2015.

\bibitem[Rudi et~al.(2020)Rudi, Marteau-Ferey, and Bach]{rudi2020global}
Alessandro Rudi, Ulysse Marteau-Ferey, and Francis Bach.
\newblock Finding global minima via kernel approximations.
\newblock In \emph{Arxiv preprint arXiv:2012.11978}, 2020.

\bibitem[Salimans et~al.(2018)Salimans, Metaxas, Zhang, and
  Radford]{salimans2018improving}
Tim Salimans, Dimitris Metaxas, Han Zhang, and Alec Radford.
\newblock Improving {GAN}s using optimal transport.
\newblock In \emph{International Conference on Learning Representations}, 2018.

\bibitem[Schiebinger et~al.(2019)Schiebinger, Shu, Tabaka, Cleary, Subramanian,
  Solomon, Gould, Liu, Lin, Berube, et~al.]{schiebinger2019optimal}
Geoffrey Schiebinger, Jian Shu, Marcin Tabaka, Brian Cleary, Vidya Subramanian,
  Aryeh Solomon, Joshua Gould, Siyan Liu, Stacie Lin, Peter Berube, et~al.
\newblock Optimal-transport analysis of single-cell gene expression identifies
  developmental trajectories in reprogramming.
\newblock \emph{Cell}, 176\penalty0 (4):\penalty0 928--943, 2019.

\bibitem[Seguy et~al.(2018)Seguy, Damodaran, Flamary, Courty, Rolet, and
  Blondel]{seguy2018maps}
Vivien Seguy, Bharath~Bhushan Damodaran, Remi Flamary, Nicolas Courty, Antoine
  Rolet, and Mathieu Blondel.
\newblock Large-scale optimal transport and mapping estimation.
\newblock In \emph{{International Conference on Learning Representations}},
  2018.

\bibitem[Steinwart and Christmann(2008)]{steinwart2008support}
Ingo Steinwart and Andreas Christmann.
\newblock \emph{Support Vector Machines}.
\newblock Springer Science \& Business Media, 2008.

\bibitem[Su et~al.(2015)Su, Wang, Shi, Zeng, Sun, Luo, and Gu]{su2015optimal}
Zhengyu Su, Yalin Wang, Rui Shi, Wei Zeng, Jian Sun, Feng Luo, and Xianfeng Gu.
\newblock Optimal mass transport for shape matching and comparison.
\newblock \emph{IEEE transactions on pattern analysis and machine
  intelligence}, 37\penalty0 (11):\penalty0 2246--2259, 2015.

\bibitem[Takatsu(2011)]{takatsu2011gaussian}
Asuka Takatsu.
\newblock {Wasserstein geometry of Gaussian measures}.
\newblock \emph{Osaka Journal of Mathematics}, 48\penalty0 (4):\penalty0 1005
  -- 1026, 2011.

\bibitem[Vacher et~al.(2021)Vacher, Muzellec, Rudi, Bach, and
  Vialard]{vacher2021dimensionfree}
Adrien Vacher, Boris Muzellec, Alessandro Rudi, Francis Bach, and
  Francois-Xavier Vialard.
\newblock A dimension-free computational upper-bound for smooth optimal
  transport estimation.
\newblock \emph{Conference on Learning Theory}, 2021.

\bibitem[Weed and Berthet(2019)]{weed2019estimation}
Jonathan Weed and Quentin Berthet.
\newblock Estimation of smooth densities in {W}asserstein distance.
\newblock In \emph{Conference on Learning Theory}, pages 3118--3119, 2019.

\bibitem[Williams and Seeger(2001)]{williams2001using}
Christopher Williams and Matthias Seeger.
\newblock Using the {N}ystr{\"o}m method to speed up kernel machines.
\newblock \emph{Advances in Neural Information Processing Systems},
  13:\penalty0 682--688, 2001.

\bibitem[Yang et~al.(2020)Yang, Damodaran, Venkatachalapathy, Soylemezoglu,
  Shivashankar, and Uhler]{yang2020predicting}
Karren~Dai Yang, Karthik Damodaran, Saradha Venkatachalapathy, Ali~C
  Soylemezoglu, GV~Shivashankar, and Caroline Uhler.
\newblock Predicting cell lineages using autoencoders and optimal transport.
\newblock \emph{PLoS computational biology}, 16\penalty0 (4):\penalty0
  e1007828, 2020.

\end{thebibliography}
